%% file: icml_2023.tex
\documentclass[nohyperref]{article}

\usepackage{microtype}
\usepackage{graphicx}
\usepackage{subcaption}
\usepackage{booktabs} %

\usepackage{hyperref}

\usepackage[accepted]{icml2023}

\usepackage{amsmath}
\usepackage{amssymb}
\usepackage{mathtools}
\usepackage{amsthm}
\usepackage{bm}

\usepackage[capitalize,noabbrev]{cleveref}

\theoremstyle{plain}
\newtheorem{theorem}{Theorem}[section]
\newtheorem*{theorem*}{Theorem}

\newtheorem{lemma}[theorem]{Lemma}
\newtheorem*{lemma*}{Lemma}
\newtheorem{corollary}[theorem]{Corollary}
\newtheorem{definition}[theorem]{Definition}
\newtheorem{assumption}[theorem]{Assumption}
\theoremstyle{plain}
\newtheorem{remark}[theorem]{Remark}

\usepackage[textsize=tiny]{todonotes}

\usepackage{xfrac}
\input{usercommands}
\input{letters}

\icmltitlerunning{Differential Privacy has Bounded Impact on Fairness in Classification}

\begin{document}

\twocolumn[
\icmltitle{Differential Privacy has Bounded Impact on Fairness in Classification}

\icmlsetsymbol{equal}{*}

\begin{icmlauthorlist}
\icmlauthor{Paul Mangold}{equal,affil}
\icmlauthor{Michaël Perrot}{equal,affil}
\icmlauthor{Aurélien Bellet}{affil}
\icmlauthor{Marc Tommasi}{affil}
\end{icmlauthorlist}

\icmlaffiliation{affil}{Univ. Lille, Inria,  CNRS, Centrale Lille, UMR 9189 - CRIStAL, F-59000 Lille, France}

\icmlcorrespondingauthor{Paul Mangold}{paul.mangold@inria.fr}
\icmlcorrespondingauthor{Michaël Perrot}{michael.perrot@inria.fr}

\icmlkeywords{Fairness, Differential Privacy, Classification, Machine Learning}

\vskip 0.3in
]

\printAffiliationsAndNotice{\icmlEqualContribution} %

\begin{abstract}
We theoretically study the impact of differential privacy on fairness in classification. We prove that, given a class of models, popular group fairness measures are pointwise Lipschitz-continuous with respect to the parameters of the model. This result is a consequence of a more general statement on accuracy conditioned on an arbitrary event (such as membership to a sensitive group), which may be of independent interest. We use this Lipschitz property to prove a non-asymptotic bound showing that, as the number of samples increases, the fairness level of private models gets closer to the one of their non-private counterparts. This bound also highlights the importance of the confidence margin of a model on the disparate impact of differential privacy.
\end{abstract}

\section{Introduction}

The performance of machine learning models have mainly been evaluated in terms of utility, that is their ability to solve specific tasks. However, they can be used in sensitive contexts, and impact people's lives. It is thus crucial that users can trust these models. While trustworthiness encompasses multiple concepts, fairness and privacy have attracted a lot of interest in the past few years. Fairness requires models not to unjustly discriminate against specific individuals or subgroups of the population, and privacy preserves individual-level information about the training data from being inferred from the model. These two notions have been extensively studied in isolation: there exists numerous approaches to learn fair models~\citep{caton2020fairness,mehrabi2021survey},
or to preserve privacy~\citep{dwork2014algorithmic,liu2021machine}. However, only few works studied the interplay between privacy and fairness. In this paper, we take a step forward in this direction, proposing a new theoretical bound on the relative impact of privacy on fairness in classification.

Fairness takes various forms (depending on the task and context), and several definitions exist. On the one hand, the goal may be to ensure that similar individuals are treated similarly. This is captured by individual fairness~\citep{dwork2012fairness} and counterfactual fairness~\citep{kusner2017counterfactual}. On the other hand, group fairness requires that decisions made by machine learning models do not unjustly discriminate against subgroups of the population.
In this paper, we focus on group fairness and consider four popular definitions, namely Equalized Odds~\citep{hardt2016equality}, Equality of Opportunity~\citep{hardt2016equality}, Accuracy Parity~\citep{zafar2017fairness}, and Demographic Parity~\citep{calders2009building}.

Differential privacy~\citep{dwork2006Differential} has been widely adopted for controlling how much information the output of an algorithm may leak about its input data. It allows publishing machine learning models while preventing an adversary from guessing
too confidently the presence (or absence) of an individual in the training data. To enforce differential privacy, one typically
releases a noisy estimate of the true model~\citep{dwork2006Differential}, so as to conceal any sensitive information
contained in individual data points. This induces a trade-off between the strength
of the protection and the utility of the learned model. While this
trade-off has been extensively studied~\citep{chaudhuri2011Differentially,bassily2014Private,liu2021machine}, its
implications for fairness are not yet well understood.

\textbf{Contributions.} In this work, we quantify the difference in fairness levels between private and non-private models in multi-class classification. We derive high probability bounds showing that this difference shrinks at a rate of $\widetilde O(\sfrac{\sqrt{p}}{n})$, where $n$ is the number of data records, and $p$ the dimension of the model. To obtain this result, we first prove that the accuracy of a model conditioned on an arbitrary event (such as membership to a sensitive group), is pointwise Lipschitz continuous with respect to the model parameters. This property is inherited by many popular group fairness notions, such as Equalized Odds, Equality of Opportunity, Accuracy Parity and Demographic Parity. Consequently, two sufficiently close models will have similar fairness levels. We then upper-bound the distance between the optimal non-private model and the private models obtained with privacy preserving mechanisms like output perturbation~\citep{chaudhuri2011Differentially,lowy2021output} or DP-SGD~\citep{song2013Stochastic,bassily2014Private}. These bounds hold for strongly convex empirical risk minimization formulations, potentially allowing explicit fairness-promoting convex regularization terms \citep{bechavod2017penalizing,huang2019stable,lohaus2020too,tran2021differentially}. Combining these two results, we derive high probability bounds on the fairness loss due to privacy. They show that, with enough training examples, (i) given an optimal non-private model, enforcing privacy will not harm fairness too much, and (ii) given a private model, the corresponding (unknown) non-private optimal model cannot be vastly fairer. Our results also highlight the role of the \emph{confidence margin} of models in the disparate impact of differential privacy: notably, if the non-private model has high per-group confidence, then our bound on the loss in fairness due to privacy will be smaller. Our contributions can be summarized as follows:
\begin{itemize}
    \item We prove that group fairness is pointwise Lipschitz, with a smaller constant for models with large margins.
    \item We bound the distance between private and optimal models, and show that the difference in their fairness levels decreases in $\widetilde O(\sfrac{\sqrt{p}}{n})$.
    \item We show that this bound can be computed even when the optimal model is unknown, and numerically demonstrate that we obtain non-trivial guarantees.
\end{itemize}

\textbf{Related work.} The joint study of fairness and privacy in machine learning only goes back a few years, and has been the focus of a recent survey \cite{survey_privacy_fairness}. One may identify three main research directions. First, it has been empirically observed that privacy can exacerbate unfairness \citep{bagdasaryan2019differential,pujol2020fair,no-private-no-fair,pate-impact,robin_hood} and, conversely, that enforcing fairness can lead to more privacy leakage for the unprivileged group \citep{chang2020privacy}. These empirical results suggest that some properties of the dataset (such as group sizes and groupwise input norms) and the choice of the private training method may affect the extent of these disparate impacts. Unfortunately, these observations are not supported by theoretical results, and it is not clear why and when disparate impact occurs.
Second, a few approaches have been proposed to learn models that are both fair and privacy preserving. However, these works either have limited theoretical guarantees on their performance~\citep{kilbertus2018blind,xu2019achieving,xu2020removing,tran2020differentially}, or learn stochastic models which might not be usable in contexts where deterministic decisions are expected~\citep{jagielski2019differentially,mozannar2020fair}. Finally, a few works have shown that fairness and privacy are incompatible in some settings, in the sense that there exists data distributions where enforcing one prevents the other from being satisfied~\citep{sanyal2022unfair}, or where enforcing both implies trivial utility~\citep{cummings2019compatibility,agarwaltrade}. While appealing at first glance, these results usually consider unrealistic cases that are hardly encountered in practice. In this paper, we also study fairness and privacy jointly but rather than studying whether they may be achieved simultaneously, we investigate the relative difference in fairness level between private and non-private models.

To the best of our knowledge, the work closest to ours is the one of~\citet{tran2021differentially}. They analyze the impact of privacy on fairness in Empirical Risk Minimization, where their notion of fairness is defined as the difference between the excess risk computed on the overall population and the excess risk computed on a subgroup of the population. They study the expected behavior over the possible private models while our results are model-specific. In line with our work, their results suggest that the distance to the decision boundary plays a key role in the disparate impact of differential privacy. However, the quantity appearing in their result is based on a second-order Taylor approximations which is loose for popular classification loss functions. In contrast, the quantity appearing in our bounds is precisely the confidence margin considered in prior work on multi-class margin-based classification \citep{cortes2013multi}.
Finally and most importantly, loss-based fairness does not necessarily imply that the actual decisions taken by the model are fair with respect to standard group-fairness notions \citep{lohaus2020too}. In contrast, our work provides guarantees in terms of these widely-accepted group fairness definitions.

\section{Preliminaries}
\label{sec:preliminaries}

In this section, we present the fairness and privacy notions that will be used throughout the paper. We consider a multi-class classification setting with a feature space $\cX$, a finite set of labels $\cY$, and a finite set $\cS$ of values for the sensitive attribute. Let $\cD$ be a distribution over $\cX\times\cS\times\cY$, and $D = \{(x_1, s_1, y_1), \dots, (x_n, s_n, y_n)\}$ be a training set of $n$ examples drawn i.i.d. from $\cD$. Let $\cH$ be a space of real-valued functions $h: \cX \times \cY \to \nsetR$ equipped with a norm $\norm{\cdot}_{\cH}$. For an example $x \in \cX$, the predicted label is the one with the highest value, that is $H(x) = \argmax_{y \in \spaceY} h(x,y)$. In case of a tie, a random label among the most likely ones is predicted.

The confidence margin \citep{cortes2013multi} of a model $h$ for an example-label pair $(x,y)$ is defined as
\begin{align*}
 \rho(h,x,y) = h(x,y) - \max_{y'\neq y} h(x,y')
 \enspace.
\end{align*}
This confidence margin is positive when the example $x$ is classified as $y$ by $h$ and negative otherwise. In this paper, we make the assumption that the margin is Lipschitz-continuous in the model $h$.

\begin{assumption}[Lipschitzness of the margin]
  \label{assum:lipschitz-constant-fct-val}
We assume that $\rho$ is Lipschitz-continuous in its first argument, that is for all $h, h' \in \cH$ and $(x,y) \in \cX \times \cY$,
\begin{align*}
    \abs{\rho(h,x,y) - \rho(h',x,y)} \leq L_{x,y}\norm{h - h'}_\cH
    \enspace,
\end{align*}
where $L_{x,y}<+\infty$ may depend on the example $(x,y)$.
\end{assumption}
Note that this assumption is not very restrictive. Typically, it holds for any class of differentiable models with either (i) bounded gradients, or (ii) continuous gradients on a compact parameter space (\eg generalized linear models or smooth deep neural networks \citep{hastie2009Elements}). We stress that $L_{x,y}$ \emph{does not need to hold uniformly on the data} (although it must hold uniformly on $\cH$). Indeed, we will see in our results (\Cref{sec:bound-fairness}) that a large $L_{x,y}$ can be compensated for by the small probability of $x, y$ in the data distribution.

To illustrate \Cref{assum:lipschitz-constant-fct-val}, consider linear models of the form $h(x,y)=W_y^T x$ where $W$ is a real-valued matrix where each line is a vector $W_y$ of label-specific parameters.
Define $\norm{h - h'}_\cH = \norm{W - W'}_2$. Then, remark that
$
    \abs{\rho(h,x,y) - \rho(h',x,y)}
    \leq{}  \abs{h(x,y) - h'(x,y)} + \max_{y' \neq y} \abs{h(x,y') - h'(x,y')}
    \leq{}  2 \norm{x}_2 \norm{h - h'}_\cH \!.$
We thus have $L_{x,y} = 2 \norm{x}_2$.%

The goal of a learning algorithm $\cA: (\cX\times\cS\times\cY)^n \rightarrow \cH$ is to find the best possible model to solve the task. In this work, the quality of a model $h$ is evaluated through its accuracy $\Acc(h) = \prob\left(H(X) = Y\right)$ but also its fairness level (as defined in Section~\ref{subsec:fairness}). Furthermore, given a non-private algorithm $\cA$, our goal will be to compare the quality of its output to that of a private version $\cA^\text{priv}$ of $\cA$ that guarantees differential privacy.

\subsection{Fairness}
\label{subsec:fairness}

In this paper, we focus on group fairness. These definitions are based on the idea that a group of individuals should not be discriminated against, compared to the overall population. Usually, these groups are defined by the sensitive attribute from $\cS$. However, in some cases, it is necessary to consider more fine grained partitions. This is for example the case in Equalized Odds~\citep{hardt2016equality}, where a model is fair if its performance is the same on the overall population and on subgroups of individuals that share the same sensitive group and the same label. Thus, for the sake of generality, we assume that the data can be partitioned into $K$ disjoint groups denoted by $D_1, \ldots, D_k, \ldots, D_K$. As in \citet{maheshwari2022fairgrad}, we consider fairness definitions that, for each group $k$, can be written~as:
\begin{align}
    F_k(h, D) = C_k^{0} + \sum_{k'=1}^K C_k^{k'}\prob\left(H(X) = Y \;\middle|\; D_{k'} \right)
    \enspace, \label{eq:fairness}
\end{align}
where the $C_k^{k'}$'s are group specific values independent of $h$, that typically depend on the size of the groups. In Appendix~\ref{app:sec:fairnessfunctions}, we show that usual group fairness notions such as Demographic Parity (with binary labels) \citep{calders2009building}, Equality of Opportunity \citep{hardt2016equality}, Equalized Odds \citep{hardt2016equality}, and Accuracy Parity \citep{zafar2017fairness} can all be expressed in the form of~\eqref{eq:fairness}. By convention, we consider that $F_k(h, D) > 0$ when the group $k$ is advantaged by $h$ compared to the overall population, $F_k(h, D) < 0$ when the group is disadvantaged and $F_k(h, D) = 0$ when $h$ is fair for group $k$.

In some cases, rather than measuring fairness for each group $k$ independently, it is interesting to summarize the information with an aggregate value. For example, we will use the mean of the absolute fairness level of each group:
\begin{align}
    \text{Fair}(h,D) = \frac{1}{K} \sum_{k=1}^K \abs{F_k(h, D)}
    \enspace,
    \label{eq:aggregatefairness}
\end{align}
which is 0 when $h$ is fair and positive when it is unfair.

\subsection{Differential Privacy}
\label{subsec:differential-privacy}

We measure the privacy of machine learning models with differential privacy (see \Cref{def:diff-priv} below). Differential privacy (DP) guarantees that the outcomes of a randomized algorithm are similar when run on datasets that differ in at most one data point. It effectively preserves privacy by preventing an adversary observing the trained model from inferring the presence of an individual in the training set.
A key property of differential privacy is that it still holds after post-processing of the algorithm's outcome \citep{dwork2006Differential}, as long as this post-processing is independent of the data.
Let $D, D' \in (\cX\times\cS\times\cY)^n$ be two datasets of $n$ elements. We say that they are \emph{neighboring} (denoted by $D \approx D'$) if
they differ in at most one element.
\begin{definition}[Differential Privacy -- \citet{dwork2006Differential}]
\label{def:diff-priv}
  Let $\cA^{\priv}: (\cX\times\cS\times\cY)^n \rightarrow \cH$ be a randomized algorithm. We say that $\cA^{\priv}$~is
  $(\epsilon,\delta)$-differentially private if, for all neighboring
  datasets $D, D' \in (\cX\times\cS\times\cY)^n$ and all subsets of hypotheses $\cH' \subseteq \cH$,
  \begin{align*}
    \prob(\cA^{\priv}(D) \in \cH')
    & \le \exp(\epsilon) \prob(\cA^{\priv}(D') \in \cH') + \delta
      \enspace.
  \end{align*}
\end{definition}
To design differentially private algorithms to estimate a function $\cA: (\cX \times \cS \times \cY)^n \rightarrow \RR^p$, we need to quantify how much changing one point in a dataset can impact the output of $\cA$. This is typically measured by (an upper bound on) the  $\ell_2$-sensitivity of $\cA$, defined as
\begin{align*}
  \Delta(\cA)
  & = \sup_{D\approx D'} \norm{\cA(D) - \cA(D')}_2
    \enspace.
\end{align*}
The value of $\cA$ on a dataset $D \in (\cX \times \cS \times \cY)^n$ can then be released privately
using the Gaussian mechanism \citep{dwork2014algorithmic}. Formally, to guarantee $(\epsilon,\delta)$-differential privacy, we add
Gaussian noise to $\cA(D)$, calibrated to its sensitivity and the desired
level of privacy:
\begin{align*}
  \cA^{\text{priv}}(D) = \cA(D) + \cN\Big(0, \frac{2\Delta(\cA)^2\log(1.25/\delta)}{\epsilon^2} \mathbb{I}_p\Big)
  \enspace,
\end{align*}
where $\cN(0,\sigma^2\mathbb{I}_p)$ is a sample from the normal
distribution with mean zero and variance $\sigma^2 \mathbb{I}_p$.
In many cases (\eg when the dataset $D$ is large), $\cA^{\text{priv}}$ is computed on a random subsample of $D$. Assuming $\cA^{\text{priv}}$ is $(\epsilon,\delta)$-differentially private, applying $\cA^{\text{priv}}$ to a randomly selected fraction $q$ of $D$ satisfies $(O(q\epsilon), q\delta)$-differential privacy, thereby amplifying privacy guarantees \citep{kasiviswanathan2011What,beimel2014Bounds}.
This privacy amplification by subsampling phenomenon, together with the Gaussian mechanism, serve as building blocks in more complex algorithms. In particular, they can be composed \citep{dwork2014algorithmic}, allowing the design of iterative private algorithms such as DP-SGD \citep{bassily2014Private,abadi2016Deep}.

\section{Pointwise Lipschitzness and Group Fairness}

Here, we show that several \emph{group fairness notions are pointwise Lipschitz} with respect to the model. To this end, we first prove a more general result on the pointwise Lipschitzness of accuracy conditionally on an arbitrary event.

\subsection{Pointwise Lipschitzness of Conditional Accuracy}
\label{sec:bound-cond-pred}

We first relate the difference of conditional accuracy of two models to the distance that separates them. This is summarized in the next theorem.
\begin{theorem}[Pointwise Lipschitzness of Conditional Accuracy]
  \label{thm:bound-on-diff-proba}
  Let $\cH$ be a set of real-valued functions with $L_{X,Y}$ the Lipschitz constants defined in
  \Cref{assum:lipschitz-constant-fct-val}. Let $h, h' \in \cH$ be two models, $(X,Y,S)$ be a triple of random variables with distribution $\cD$, and $E$ be an arbitrary
  event. Assume that $\expect \left(\sfrac{L_{X,Y}}{\abs{\rho(h',X,Y)}} \;\middle|\; E \right) < +\infty$, then
  \begin{align}
    \abs{ \prob( H(X) = Y \;\middle|\;E ) - \prob( H'(X) = Y \;\middle|\;E ) } \negspace{14em}&
    \nonumber
    \\
        \leq{}& \expect\left(\frac{L_{X,Y}}{\abs{\rho(h,X,Y)}} \;\middle|\; E \right) \norm{h - h'}_{\cH} \enspace.
    \tag{Lip}
    \label{eq:thm:pointwise-lipschitz}
  \end{align}
\end{theorem}
\begin{proof} (Sketch) The proof of this theorem is in two steps. First, we use the Lipschitzness of the margin (\Cref{assum:lipschitz-constant-fct-val}), the triangle inequality, and the union bound to show that $\abs{\prob\left(H(X) = Y \;\middle|\;E\right) - \prob\left(H'(X) = Y \;\middle|\; E\right)} \leq{} \prob\left(\sfrac{L_{X,Y}}{\abs{\rho(h,X,Y)}}  \geq \sfrac{1}{\norm{h - h'}_\cH} \;\middle|\; E\right)$. Then, applying Markov's inequality gives the desired result. The complete proof can be found in \Cref{sec-app:proof-thm-1}.
\end{proof}

\Cref{thm:bound-on-diff-proba} shows the pointwise lipschitzness of the function $h \mapsto \prob\left(H(X) = Y \;\middle|\; E\right)$ and underlines the importance of the confidence margin $\rho(h,x,y)$ of the model $h$. Note that $\sfrac{L_{x,y}}{\abs{\rho(h,x,y)}}$ is small when the model $h$ is confident (relatively to $L_{x,y}$) in its prediction for the true label $y$. This implies that, when the probability (given $E$) that a point has a small margin (relatively to $L_{x,y}$) is small, $\expect\left(\frac{L_{X,Y}}{\abs{\rho(h,X,Y)}} \;\middle|\; E \right)$ is also small. This is notably the case for large margin classifiers.

Additionally, we note that data records that are unlikely do not affect the result too much. Indeed, even if $\sfrac{L_{x,y}}{\abs{\rho(h,x,y)}}$ is large, this can be compensated for by the small probability of observing $x,y$ so that the value of $\expect\left(\frac{L_{X,Y}}{\abs{\rho(h,X,Y)}} \;\middle|\; E \right)$ is not significantly affected.

It is worth noting that the bound presented in \Cref{thm:bound-on-diff-proba} can be tightened (at the expense of readability) without affecting the quantities that need to be controlled, that is the margin $\abs{\rho(h,x,y)}$ and the distance $\norm{h-h'}_{\cH}$. Hence, note that given $(x,y) \in \cX\times\cY$, if $\abs{\rho(h,x,y)}  \geq L_{x,y}\norm{h - h'}_\cH$, then it means that $h$'s margin is large enough to ensure that $h$ and $h'$ have the same prediction on $x$. The corresponding term in the expectation may then be accounted for as zero, improving the upper bound (\Cref{app:rem:large-enough-margins}). Interestingly, if all the examples are classified with such a large margin, our bound becomes $0$, further hinting toward the importance of large margin classifiers. This result may be further tightened by using a Chernoff bound instead of Markov's inequality (\cref{app:rem:chernoff-instead-markov}), yielding
    $\abs{ \prob( H\!(X) \!=\! Y |E ) \!-\! \prob( H'\!(X) \!=\! Y | E ) } \le \beta_{\!X,Y}(h)$, with
    \begin{align*}
        \beta_{X,Y}(h)=\inf_{t \ge 0} \Big\{ e^{t\norm{h-h'}_\cH}
    \expect\Big(e^{-\frac{t\abs{\rho(h,X,Y)}}{L_{X,Y}}} \Big| E \Big)
    \Big\}
    \enspace.
    \end{align*}
    In the subsequent theoretical developments, we use the bound derived in \Cref{thm:bound-on-diff-proba} for the sake of readability. In the numerical experiments (\Cref{sec:experiments}), we use the version of the bound that yields the tightest results by combining both of the aforementioned techniques.

\subsection{Pointwise Lipschitzness of Group Fairness Notions}
\label{sec:bound-fairness}

We now use \Cref{thm:bound-on-diff-proba}'s general result to relate the fairness levels
of two classifiers, based on their distance.
In \Cref{thm:bound-on-diff-fairness}, we show that fairness notions in the form of~\eqref{eq:fairness} are pointwise Lipschitz.
\begin{theorem}[Pointwise Lipschitzness of Fairness\label{thm:bound-on-diff-fairness}]
  Let $h, h' \in \cH$, and $L_{X,Y}$ defined as in
  \Cref{assum:lipschitz-constant-fct-val}. For any fairness notion of the form of~\eqref{eq:fairness}, we have, for all $k \in [K]$,
  \begin{align*}
      \abs{F_k(h,D) - F_k(h',D)} \leq \chi_k(h, D)\norm{h - h'}_{\cH}
      \enspace.
  \end{align*}
  with $ \chi_k(h, D) = \sum_{k'=1}^K \abs{C_k^{k'}}\expect \left( \frac{L_{X,Y}}{\abs{\rho(h,X,Y)}} \;\middle|\; D_{k'} \right)$.
  Similarly, for the aggregate measure of fairness defined in~\eqref{eq:aggregatefairness},
  \begin{align*}
      \!\abs{\text{Fair}(h,D) \!-\! \text{Fair}(h',D)} \!\leq \!\frac{1}{K} \! \sum_{k=1}^K \chi_k(h, D)\norm{h - h'}_{\cH}
      \enspace.
  \end{align*}
\end{theorem}
\begin{proof} (Sketch) To prove the first claim, we use the triangle inequality to show that, for each group, the absolute difference in fairness is bounded by a combination of absolute differences between conditional probabilities. We can then apply Theorem~\ref{thm:bound-on-diff-proba}. The second claim follows by applying the first one to each group independently. The complete proof is provided in Appendix~\ref{app:sec:proof-bound-on-diff-fairness}.
\end{proof}

\Cref{thm:bound-on-diff-fairness} implies that \emph{classifiers that are
sufficiently close have similar fairness
levels}. This has two major consequences when studying a given model. On the one hand, we have an upper bound on the harm that can be done to fairness: small variations of the model cannot make it much more unfair. On the other hand, we have a lower bound on the distance needed to make a model fair: making the model significantly more fair requires to substantially alter it.
In the next corollary, we instantiate \Cref{thm:bound-on-diff-fairness} for various popular group fairness notions, and for accuracy.
\begin{corollary}%
\label{cor:bound-on-popular-fairness}
  Let $h, h' \in \cH$, and $L_{X,Y}$ defined as in
  \Cref{assum:lipschitz-constant-fct-val}. The difference in fairness or accuracy between $h$ and $h'$ can be bounded as follows.%

\textbf{Equalized Odds \citep{hardt2016equality}}: the data is divided into $K=\abs{\cY \times \cS}$ groups such that for all $(y,r) \in \cY \times \cS,$
\begin{align*}
  \chi_{(y,r)}(h,D) ={}& \expect \left( \tfrac{L_{X,Y}}{\abs{\rho(h,X,Y)}} \;\middle|\; Y = y \right)\\*
  &+ \expect \left( \tfrac{L_{X,Y}}{\abs{\rho(h,X,Y)}} \;\middle|\; Y = y, S = r \right)
  \enspace.
\end{align*}
\textbf{Equality of Opportunity \citep{hardt2016equality}}: we let $\cY' \subseteq \cY$ the set of desirable outcomes. The data is divided into $K=\abs{\cY \times \cS}$ such that for all $(y,r) \in \cY \times \cS,$
\begin{align*}
    \chi_{(y,r)}(h,D) =
    \expect \left( \tfrac{L_{X,Y}}{\abs{\rho(h,X,Y)}} \;\middle|\; Y = y, S = r \right) \\ + \expect \left( \tfrac{L_{X,Y}}{\abs{\rho(h,X,Y)}} \;\middle|\; Y = y \right)
    \enspace,
\end{align*}
if $y$ is a desired outcome, and
$\chi_{(y,r)}(h,D) = 0$ otherwise.

\textbf{Accuracy Parity \citep{zafar2017fairness}}: the data is divided into $K=\abs{\cS}$ groups such that for all $r \in \cS$,
\begin{align*}
    \chi_{(r)}(h,D) = \expect \left( \tfrac{L_{X,Y}}{\abs{\rho(h,X,Y)}} \right) + \expect \left( \tfrac{L_{X,Y}}{\abs{\rho(h,X,Y)}} \;\middle|\; S = r \right)
  \enspace.
\end{align*}
\textbf{Demographic Parity (Binary Labels) \citep{calders2009building}}: the data is divided into $K=\abs{\cY \times \cS}$ groups such that for all $(y,r) \in \cY \times \cS$,
\begin{align*}
       \chi_{(y,r)}(h,D) \!=\! \expect \left( \tfrac{L_{X,Y}}{\abs{\rho(h,X,Y)}} \right) \!+\! \expect \left( \tfrac{L_{X,Y}}{\abs{\rho(h,X,Y)}} \;\middle|\; S = r \right)
  \enspace.
  \end{align*}
  \textbf{Accuracy}: the data is in a single group, such that
    \begin{align*}
    \chi(h,D)
    = \expect \left( \tfrac{L_{X,Y}}{\abs{\rho(h,X,Y)}} \right)
  \enspace.
  \end{align*}
\end{corollary}

\begin{proof}
This corollary follows from \Cref{thm:bound-on-diff-fairness} by replacing the $C_k^{k'}$'s by their appropriate values (depending on the considered notion). See \Cref{app:sec:fairnessfunctions} for more details.
\end{proof}

\Cref{cor:bound-on-popular-fairness} shows that our results are applicable to several \emph{group fairness notions}, but also to \emph{accuracy}. Importantly, the pointwise Lipschitz constant $\chi_{k}(h,D)$ depends both on the group $k \in [K]$, and on the considered fairness notion. This sheds light on the fact that comparing the fairness levels of two models requires special attention, as there may be important disparities depending on the considered sensitive group or fairness notion.
We will use this result in Section~\ref{sec:privacy}, where we bound the disparate impact of privacy by bounding the difference in fairness levels between private and non-private models.

\paragraph{Finite sample analysis.} In practice, it is often assumed that one does not have access to the true distribution $\cD$ but rather to a finite sample $D = \{(x_1, s_1, y_1), \dots, (x_n, s_n, y_n)\}$ of size $n$. An empirical estimate of the expectation from a finite sample is then defined as $\widehat{\expect}\left(f(X)\right) = \frac{1}{n} \sum_{i=1}^n f(x_i)$. The results presented in \Cref{thm:bound-on-diff-proba}, \Cref{thm:bound-on-diff-fairness}, and \Cref{cor:bound-on-popular-fairness} also hold in this finite sample setting. For instance, denoting by $\widehat{\chi}_k$ the empirical version of $\chi_k$, we have that $\forall k \in [K]$,
\begin{align*}
    \abs{\widehat{F}_k(h,D) - \widehat{F}_k(h',D)} \leq \widehat{\chi}_k(h, D)\norm{h - h'}_{\cH}
    \enspace.
\end{align*}
One may then wonder whether it is possible to connect the true fairness of a model $h$ to the empirical fairness of a second model $h'$, that is bound $\abs{F_k(h,D) - \widehat{F}_k(h',D)}$. In the next lemma, we show that such bound can indeed be obtained when $h$ and $h'$ were learned on $D$.
\begin{lemma}
\label{lem:finitesample}
Let $D$ be a finite sample of $n \geq \frac{8\log\left(\frac{2K+1}{\delta}\right)}{\min_{k'}p_{k'}}$ examples drawn i.i.d. from $\cD$, where $p_{k'}$ is the true proportion of examples from group $k'$. Assume that $\prob_{D \sim \cD^n}\big(\sum_{k'=0}^K\babs{C_k^{k'} - \widehat{C}_k^{k'}} > \alpha_{C}\big) \leq B_3\exp\left(-B_4\alpha_{C}^2n\right)$. Let $\cH$ be an hypothesis space and $d_{\cH}$ be the Natarajan dimension of $\cH$. With probability at least $1-\delta$ over the choice of $D$, $\forall h, h' \in \cH$
    \begin{align*}
        \abs{F_k(h,D) - \widehat{F}_k(h',D)} \leq \widehat{\chi}_k(h, D)\norm{h - h'}_{\cH} \negspace{14em}&\\
        &+ \widetilde{O}\left(\sum_{k'=1}^K \abs{\widehat{C}_k^{k'}}\sqrt{\frac{d_{\cH} + \log\left(\sfrac{K}{\delta}\right)}{np_{k'}}}\right)
        \enspace.
    \end{align*}
\end{lemma}
\begin{proof}
(Sketch) This lemma follows from bounding the two terms in the following inequality:
\begin{align*}
    \abs{F_k(h',D) - \widehat{F}_k(h,D)} \leq{}&\abs{\widehat{F}_k(h',D) - \widehat{F}_k(h,D)} \\
    &+ \abs{F_k(h',D) - \widehat{F}_k(h',D)}
    \enspace.
\end{align*}
The first term can be bounded using the empirical counterpart of \Cref{thm:bound-on-diff-fairness}. The second term is then bounded with high probability using standard uniform convergence bounds \citep{shalev2014understanding}.
The complete proof can be found in \Cref{app:sec:proof-finitesample} where the result is also extended to the simpler case where $h$ and $h'$ are fixed rather than learned on $D$.
\end{proof}

\section{Bounding the Relative Fairness of Private Models}
\label{sec:privacy}

In this section, we quantify the difference of fairness between a private model and its non-private counterpart. Let
$\ell: \cH \times \cX \times \cS \times \cY \rightarrow \RR$ be a loss function. Assume $\ell$ is $\Lambda$-Lipschitz, and
$\mu$-strongly-convex with respect to its first variable. Assume the norm $\norm{\cdot}_\cH$ is Euclidean, and that $\cH$ is convex. We define the optimal model $h^* \in \cH$ as
\begin{align}
  \label{pb:dp-erm}
  h^* = \argmin_{h \in \cH} f(h) = \frac{1}{n} \sum_{i=1}^n \ell(h; x_i, s_i, y_i)
  \enspace.
\end{align}

Note that, since $f$ is strongly-convex, $h^*$ is unique, and is thus a natural choice for the non-private reference model. This strong convexity assumption could be relaxed: this would require to define another reference model (e.g., the output of non-private SGD) and to bound the distance between this reference model and the private model.

Two mechanisms are commonly used to find a differentially private approximation $h^{\priv}$ of $h^*$: output perturbation \citep{chaudhuri2011Differentially,lowy2021output},
and DP-SGD \citep{bassily2014Private,abadi2016Deep}. For both mechanisms, the distance $\norm{h^\priv - h^*}_\cH$ can be upper bounded with high probability. In this section, we recall these two mechanisms and the corresponding high probability upper bounds. We then plug these bounds in \Cref{thm:bound-on-diff-fairness} to
bound the fairness level of the private solution $h^{\priv}$
relatively to the one of the true solution $h^*$.

\subsection{Bounding the Distance between Private and Optimal Classifiers}
\label{sec:two-priv-mech}

\paragraph{Output perturbation.}
Output perturbation computes the
non-private solution $h^*$ of~\eqref{pb:dp-erm}, and releases a private estimate by the Gaussian mechanism:
\begin{align*}
  h^{\priv} = \pi_{\cH}(h^* + \cN(\sigma^2\mathbb{I}_p))
  \enspace,
\end{align*}
where $\pi_\cH$ is the projection on $\cH$. Let $\Delta$ be the
sensitivity of the function $D \mapsto \argmin_{w\in\cH} f(w;
D)$. In our setting, we have $\Delta = \sfrac{2\Lambda}{\mu n}$. Then, given $0 < \epsilon,\delta < 1$, $h^{\priv}$ is $(\epsilon, \delta)$-differentially private as long as $\sigma^2 \ge
\sfrac{2\Delta^2\log(1.25/\delta)}{\epsilon^2}$. We bound
the distance between $h^{\priv}$ and $h^*$ with high
probability in \Cref{lemma:sensitivity-output-perturbation}. %
\begin{lemma}
  \label{lemma:sensitivity-output-perturbation}
  Let $h^{\priv}$ be the vector released by output
  perturbation with noise
  $\sigma^2 = \sfrac{8\Lambda^2\log(1.25/\delta)}{\mu^2 n^2 \epsilon^2}$,
  and $0 < \zeta < 1$, then with probability at least $1 - \zeta$,
  \begin{align*}
    \norm{h^{\priv} - h^*}_2^2
      \le \frac{32p\Lambda^2\log(1.25/\delta)\log(2/\zeta)}{\mu^2 n^2 \epsilon^2}
      \enspace.
  \end{align*}
\end{lemma}

\paragraph{DP-SGD.} DP-SGD starts from some $h^0 \in \cH$ and updates it using
stochastic gradients. That is, with $\gamma > 0$, $i \sim \cU([n])$, and $\eta^t \sim \cN(0, \sigma^2 \mathbb{I}_p)$, we iteratively update
\begin{align*}
  h^{t+1} = \pi_{\cH} (h^t - \gamma(\nabla \ell(h^t; x_i, y_i) + \eta^t))
  \enspace.
\end{align*}
After $T>0$ iterations, we release $h^{\priv} = h^T$. Given $0 < \epsilon,\delta < 1$, $h^{\priv}$ is $(\epsilon,\delta)$-differentially
private when
$\sigma^2 \ge \sfrac{64 \Lambda^2 T^2 \log(3T/\delta) \log(2/\delta) }{ n^2
  \epsilon^2 }$. We bound the distance between
$h^{\priv}$ and $h^*$ with high probability in
\Cref{lemma:sensitivity-dp-sgd}.
\begin{lemma}
  \label{lemma:sensitivity-dp-sgd}
  Let $h^{\priv}$ be the vector released by DP-SGD with
  $\sigma^2 = \sfrac{64 \Lambda^2 T^2 \log(3T/\delta) \log(2/\delta) }{ n^2
  \epsilon^2 }$. Assume that the loss function is smooth in its first parameter, and that
  $\sigma_*^2 = \expect_{i\sim[n]} \norm{\nabla \ell(h^*; x_i, y_i)}^2
  \le \sigma^2$. Let $0 < \zeta < 1$, then with probability at least
  $1 - \zeta$,
  \begin{align*}
    \norm{h^{\priv} - h^*}_2^2
    & = \widetilde O\left(  \frac{p \Lambda^2 \log(1/\delta)^2}{\zeta \mu^2 n^2 \epsilon^2} \right)
      \enspace,
  \end{align*}
  where $\widetilde O$ ignores logarithmic terms in $n$ (the number of examples) and $p$ (the number of model parameters).
\end{lemma}

We prove \Cref{lemma:sensitivity-dp-sgd,lemma:sensitivity-output-perturbation} in~\Cref{sec-app:bound-output-pert,app:sec:convergence-dp-sgd}. These lemmas give upper bounds on the distance between the optimal model and the private models learned by output perturbation and DP-SGD. This distance decreases as the number of records $n$ or the privacy budget $\epsilon$ increase. Conversely, it increases with the complexity of the model (\ie with larger dimension $p$), and with the Lipschitz constant of the loss $\Lambda$. %

\begin{remark}
  For clarity of exposition in \Cref{lemma:sensitivity-dp-sgd}, we did not use minimal
  assumptions and used the simplest variant of DP-SGD. Notably, the assumption on $\sigma_*$ can be removed by
  using variance reduction schemes, and tighter bounds on $\sigma$ can
  also be obtained using Rényi Differential Privacy
  \citep{mironov2017Renyi}.
  Similarly, the assumption $\epsilon < 1$ is only used to give simple closed-form bounds. Strong convexity and smoothness assumptions can be relaxed as well.
\end{remark}

\begin{table*}[htb]
    \small
    \centering
    \caption{Upper bound, with 99\% probability, on the difference of fairness between private and
      non-private models for different fairness measures and
      accuracy. Privacy parameters are $\epsilon=1$ and $\delta=1/n^2$
      where $n$ is the number of samples in the training data.}
    \label{table:value-of-upperbound}
    \csvautobooktabularcenter
    [
    separator=semicolon,
    table head= \\ \toprule Dataset & Equality of Opportunity & Equalized Odds & Demographic Parity & Accuracy Parity & Accuracy \\\midrule]
    {rsc/bounds.csv}
\end{table*}

\subsection{Bounding the Fairness of Private Models}
\label{sec:bound-fairn-priv}

We now state our central result
(\Cref{thm:bound-on-fairness-private-models}), where we bound the fairness of $h^{\priv}$ relatively to the one of $h^*$.
\begin{theorem}
  \label{thm:bound-on-fairness-private-models}
  Let $h^*$ be the solution of~\eqref{pb:dp-erm}, and $h^{\text{priv}}$ its private estimate obtained by output perturbation.
  Let $h^{\refer} \in \{h^{\priv}, h^*\}$, and $0 < \zeta < 1$. Then, the difference of fairness of group $k \in [K]$ satisfies, with probability at least $1-\zeta$,
  \begin{align*}
    & \abs{F_k(h^{\priv},D) - F_k(h^*,D)} \\
    & \quad \le
      \frac{\chi_k(h^{\refer}, D) L\Lambda\sqrt{32p\log(1.25/\delta)\log(2/\zeta)}}{\mu n \epsilon}
      \enspace.
  \end{align*}
  \vbox{
  Similarly, if $h^{\priv}$ is estimated through DP-SGD, we
  have that, with probability at least $1 - \zeta$,
  \begin{align*}
    & \abs{F_k(h^{\priv},D) - F_k(h^*,D)} \\
    & \quad \le
      \widetilde O \left(\frac{\chi_k(h^{\refer}, D) L \Lambda \sqrt{p\log(1/\delta)}}{\sqrt{\zeta} \mu n \epsilon} \right)
      \enspace,
  \end{align*}}
  where $\widetilde O$ ignores logarithmic terms in $n$ (the number of examples) and $p$ (the number of model parameters).
\end{theorem}
\begin{proof}
By \Cref{lemma:sensitivity-output-perturbation} or \Cref{lemma:sensitivity-dp-sgd}, we control the distance $\norm{h^{\text{priv}} - h^*}$. Plugging this bound in
\Cref{thm:bound-on-diff-fairness} gives the result.
\end{proof}
This result shows that, when learning a private model, \emph{the unfairness due to
privacy vanishes at a $\widetilde{O}(\sfrac{\sqrt{p}}{n})$ rate}. To the best of our knowledge, our result is the first to quantify this rate. Importantly, \emph{it highlights the role of the confidence margin} of the classifier on the disparate impact of differential privacy. This is in line with previous empirical and theoretical work that identified the groupwise distances to the decision boundary as an important factor \citep{tran2021differentially,survey_privacy_fairness}. However, our bounds are the first to quantify this impact through a classic notion of confidence margin studied in learning theory \citep{cortes2013multi}.

Our result may be interpreted and used in various ways. A first example is the case where the private model is known but its optimal non-private counterpart is not. There, our result guarantees that, given enough examples, the fairness level of the private model is close to the one of the optimal non-private model. This allows the practitioner to give guarantees on the model, that the end user can trust. A second example is the case where the true model $h^*$ is owned by someone who cannot share it, due to privacy concerns. Imagine that the model needs to be audited for fairness. Then, the model owner can compute a private estimate of their model, and send it to the (honest but curious) auditing company. The bound allows to obtain fairness bounds for the true model from the inspection of the private one, and thus acts as a certificate of correctness of the audit done on the private version of the model.

\begin{remark}
\label{rmq:fairness-and-privacy-algorithm}
    The fairness guarantee for the private model given by \Cref{thm:bound-on-fairness-private-models} is relative to the fairness of the optimal model $h^*$, which may itself be quite unfair. A standard approach to promote fair models is to use convex relaxations of fairness as regularizers to the ERM problem \citep{bechavod2017penalizing,huang2019stable,lohaus2020too}. Interestingly, to be able to use output perturbation, we only require the objective function of~\eqref{pb:dp-erm} to be strongly convex and Lipschitz over $h \in \cH$, which is the case for these relaxations when they are combined with a squared $\ell_2$-norm. For binary classification with two sensitive groups, \citet{lohaus2020too} proved that, with a proper choice of regularization parameters, this approach can yield a fair $h^*$ (see their Theorem 1 for more details). Combined with our results, this paves the way for the design of algorithms that learn provably private and fair classifiers.
    However, several crucial challenges remain to make this approach work in practice, such as (i) finding the appropriate regularization parameters privately, and (ii) providing guarantees on the resulting classifiers' accuracy. We leave this for future work.
\end{remark}

\begin{figure*}[h]
  \centering
  \begin{subfigure}{\textwidth}
    \centering
    \includegraphics[width=0.9\linewidth]{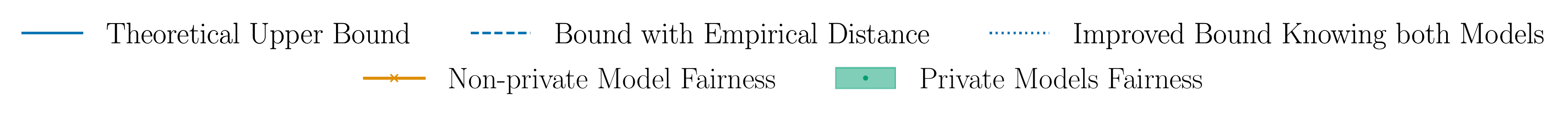}
    \vspace{-0.5em}
  \end{subfigure}
  \begin{subfigure}{0.24\textwidth}
    \includegraphics[width=\linewidth]{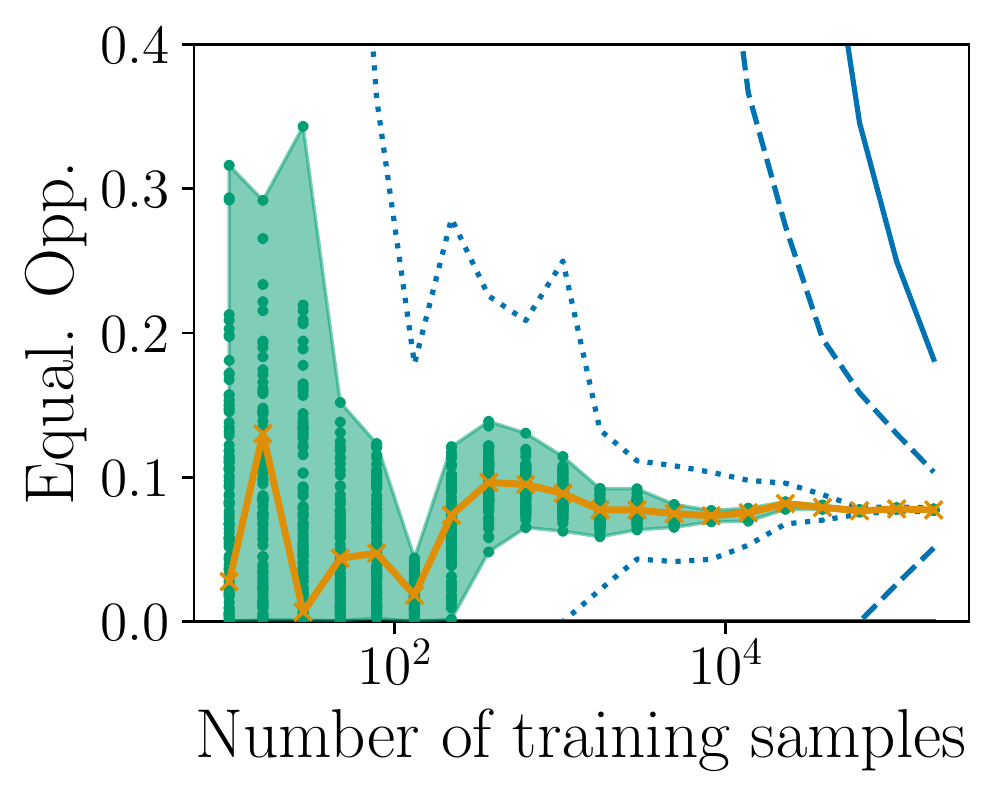}
    \caption{Equal. Opp. (\celebA)}
    \label{fig:fairness-fct-n-f1}
  \end{subfigure}
  \hfill
  \begin{subfigure}{0.24\textwidth}
    \includegraphics[width=\linewidth]{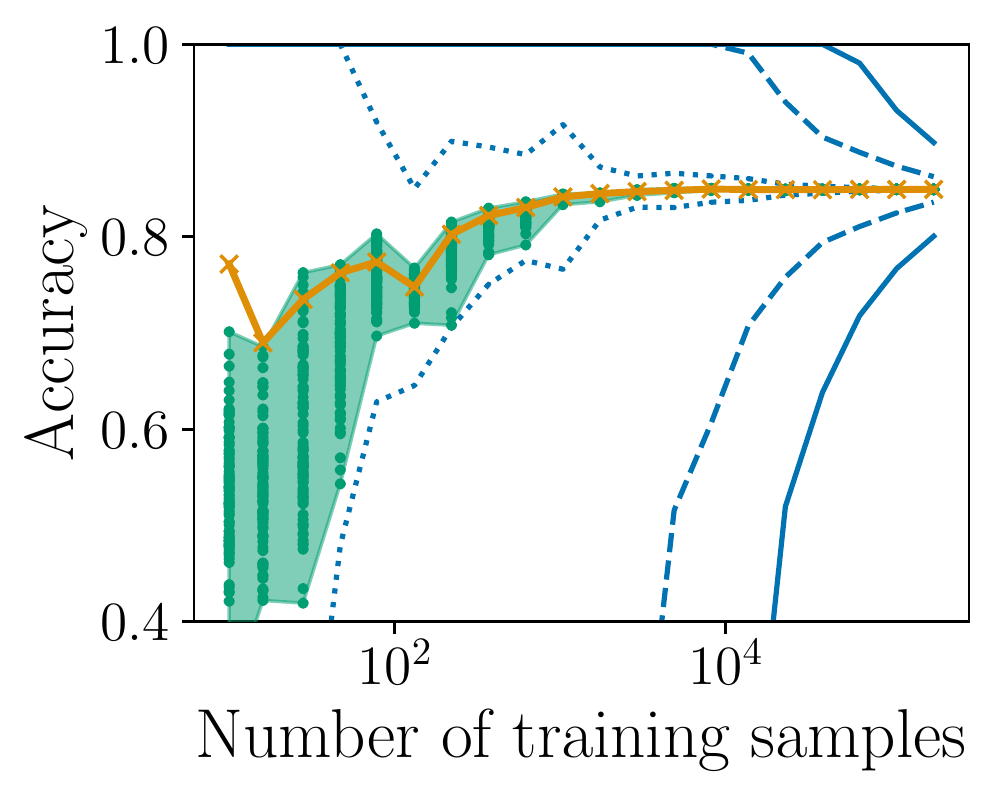}
    \caption{Accuracy (\celebA)}
    \label{fig:fairness-fct-n-a1}
  \end{subfigure}
  \hfill
  \begin{subfigure}{0.24\textwidth}
    \includegraphics[width=\linewidth]{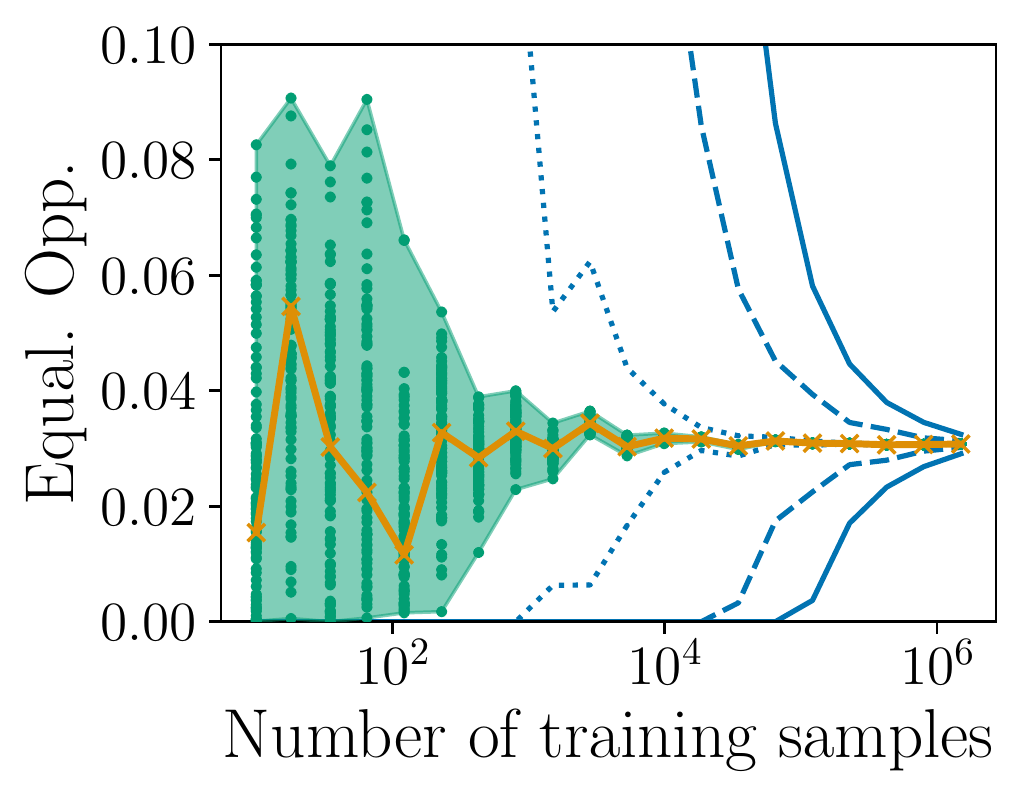}
    \caption{Equal. Opp. (\folktables)}
    \label{fig:fairness-fct-n-f2}
  \end{subfigure}
  \hfill
  \begin{subfigure}{0.24\textwidth}
    \includegraphics[width=\linewidth]{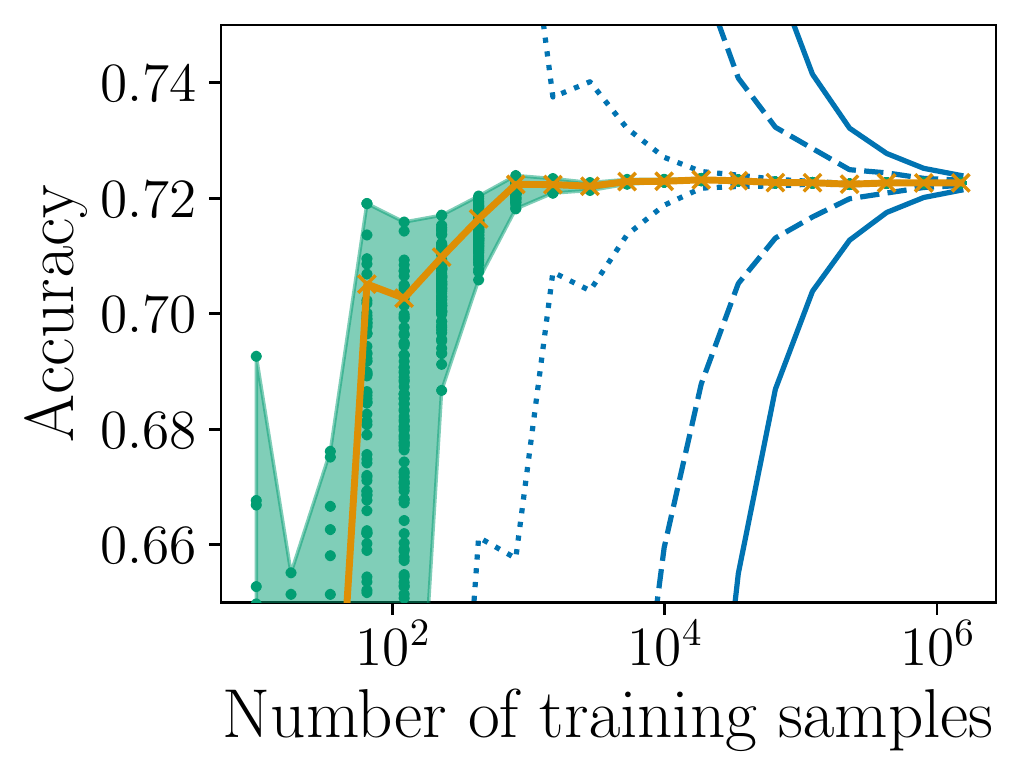}
    \caption{Accuracy (\folktables)}
    \label{fig:fairness-fct-n-a2}
  \end{subfigure}

  \centering
  \begin{subfigure}{0.24\textwidth}
    \includegraphics[width=\linewidth]{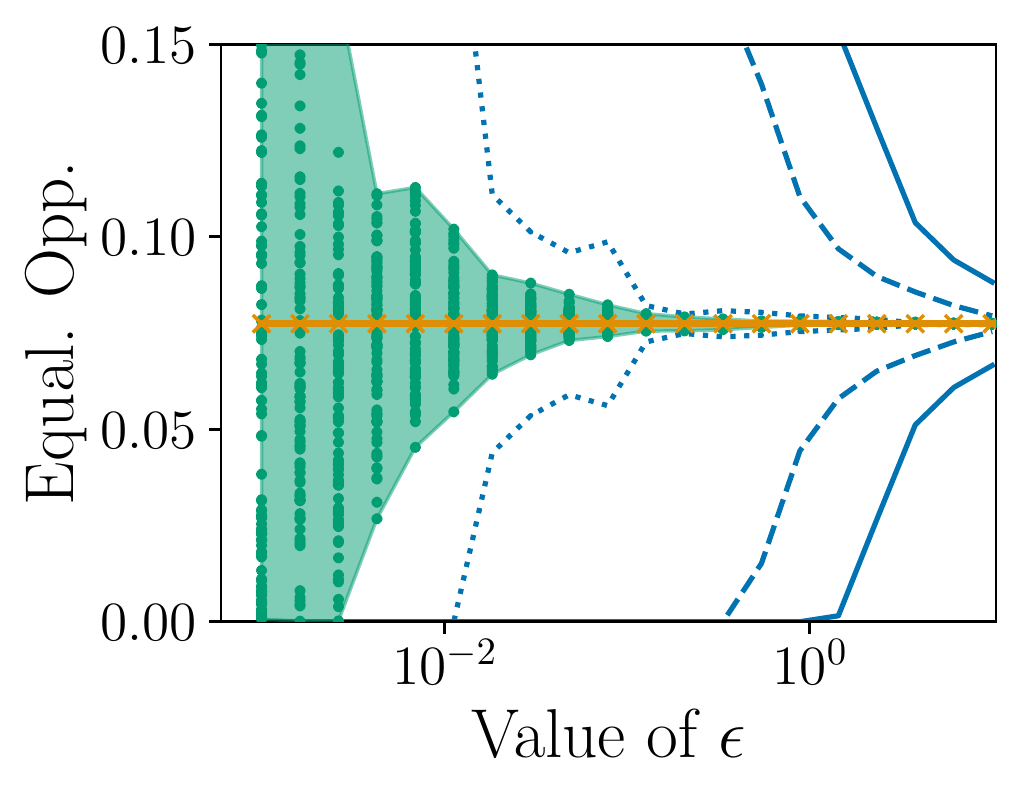}
    \caption{Equal. Opp. (\celebA)}
    \label{fig:fairness-fct-epsilon-f1}
  \end{subfigure}
  \hfill
  \begin{subfigure}{0.24\textwidth}
    \includegraphics[width=\linewidth]{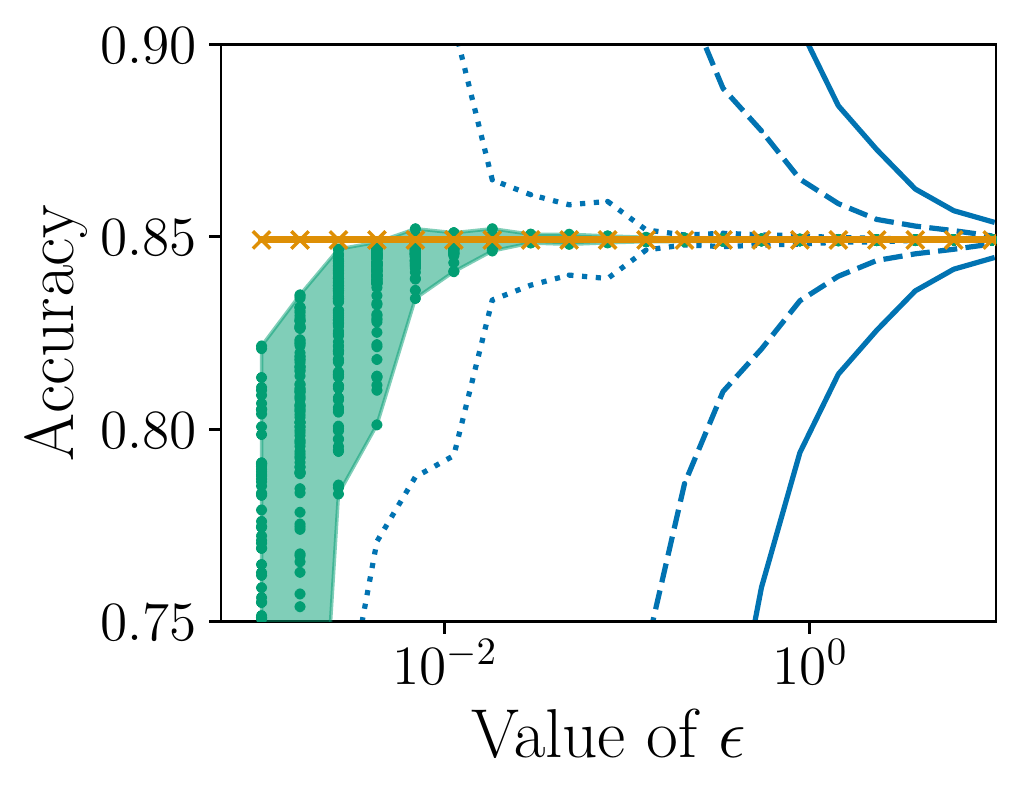}
    \caption{Accuracy (\celebA)}
    \label{fig:fairness-fct-epsilon-a1}
  \end{subfigure}
  \hfill
  \begin{subfigure}{0.24\textwidth}
    \includegraphics[width=\linewidth]{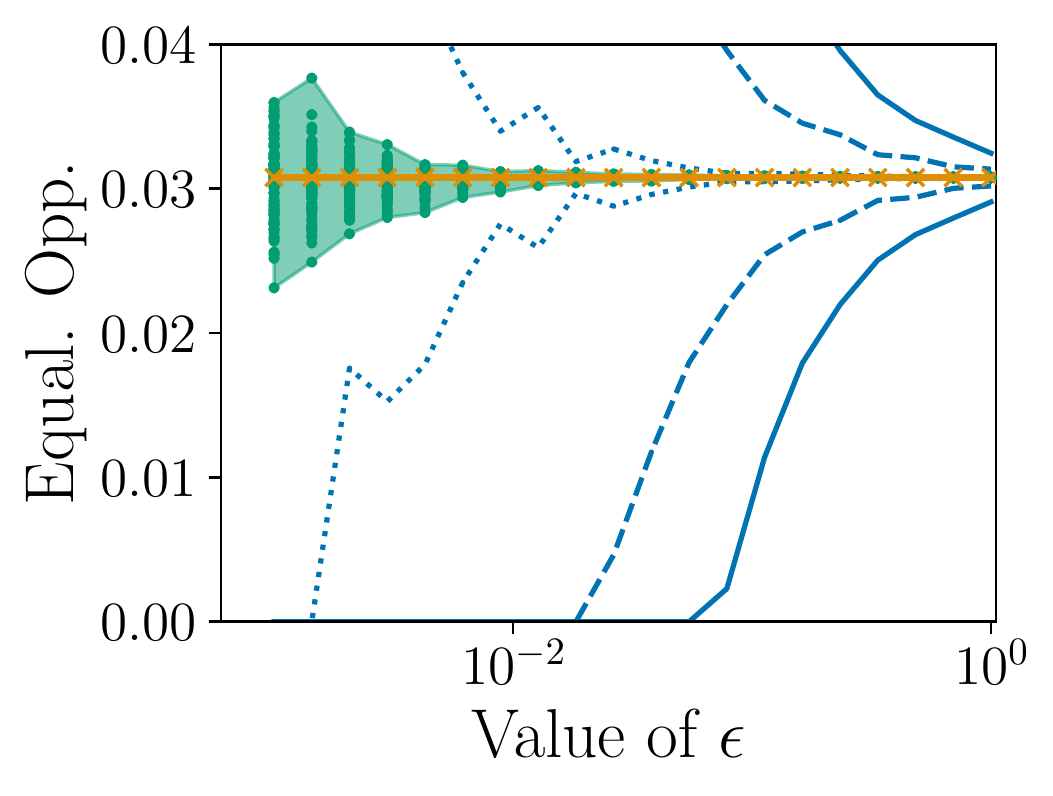}
    \caption{Equal. Opp. (\folktables)}
    \label{fig:fairness-fct-epsilon-f2}
  \end{subfigure}
  \hfill
  \begin{subfigure}{0.24\textwidth}
    \includegraphics[width=\linewidth]{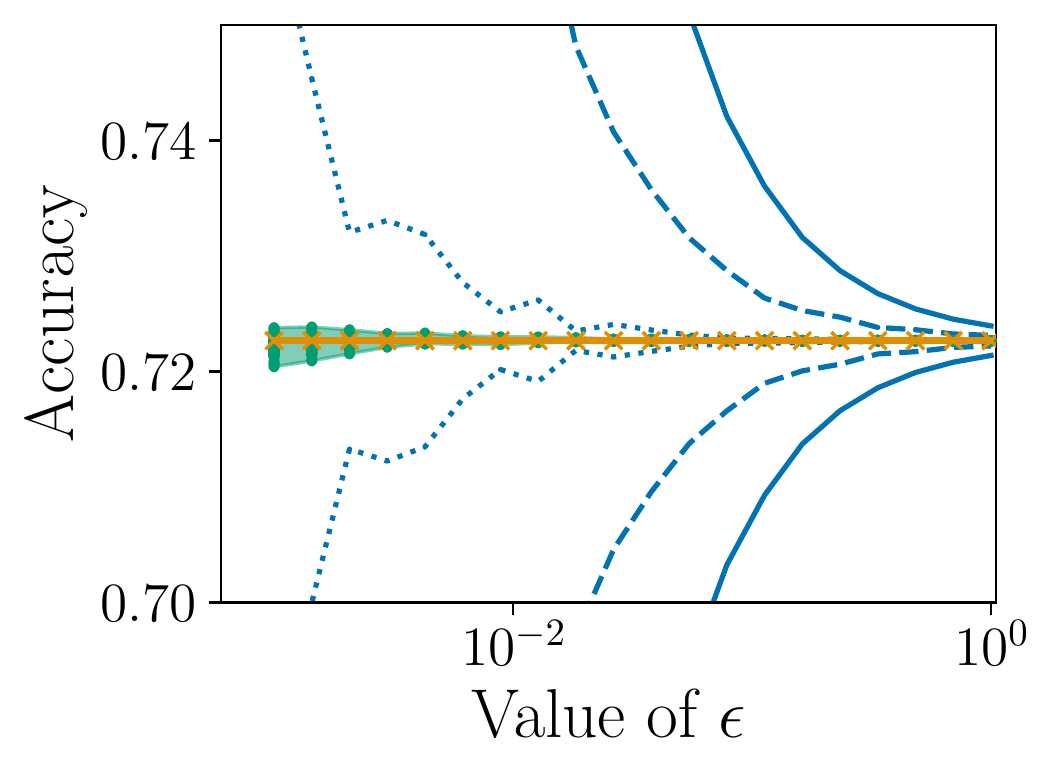}
    \caption{Accuracy (\folktables)}
    \label{fig:fairness-fct-epsilon-a2}
  \end{subfigure}
  \caption{Equality of opportunity (Equal. Opp.) and Accuracy levels for optimal non-private model
    and random private ones as a function of the number of training records $n$ (first line, with $\epsilon=1$ and $\delta=1/n^2$) and of the privacy budget
    $\epsilon$ (second line, using all available training records). For each value of $n$ and $\epsilon$, we sample $100$ private
    models and take their minimum and maximum fairness/accuracy values
    to mark the area of attainable values. The solid blue line gives the theoretical guarantees from \Cref{thm:bound-on-fairness-private-models}, while the dashed and dotted line give finer bounds when more information is available (see \Cref{sec:tightness-bound} for details).}
  \label{fig:fairness-fct-n-epsilon}
\end{figure*}

\section{Numerical Experiments}
\label{sec:experiments}

In this section, we numerically illustrate the upper bounds from
\Cref{sec:bound-fairn-priv}. We use the \celebA \citep{liu2015faceattributes} and \folktables \citep{ding2021retiring}
datasets, which respectively contain $202,599$ and $1,664,500$ samples,
with $39$ and $10$ features (including one sensitive attribute, sex, that is not not
used for prediction), and binary labels. For each dataset, we use
$90\%$ of the records for training, and the remaining $10\%$ for
empirical evaluation of the bounds. We train $\ell_2$-regularized
logistic regression models, ensuring that the underlying optimization
problem is $1$-strongly-convex. This allows learning private models by
output perturbation, for which the bound from
\Cref{thm:bound-on-fairness-private-models} holds. \footnote{The code is available at \url{https://github.com/pmangold/fairness-privacy}.}

In \Cref{sec:value-upper-bounds}, we show that we obtain non-trivial guarantees on the private model's fairness and accuracy. Then, we study the influence of the number of
training samples and of the privacy budget $\epsilon$ in \Cref{sec:influence-n-epsilon}, and discuss the tightness of our result in \Cref{sec:tightness-bound}.

\subsection{Value of the Upper Bounds}
\label{sec:value-upper-bounds}

In \Cref{table:value-of-upperbound}, we compute the value of
\Cref{thm:bound-on-fairness-private-models}'s bounds. We learn a non-private $\ell_2$-regularized logistic regression model, and use it to compute the bounds (averaged over the two groups) for multiple fairness and accuracy measures on two datasets.
In all cases, our results give non-trivial guarantees on the difference of fairness: it is bounded by at most $0.105$ for \celebA and $0.0026$ for \folktables. This means that any $(1,1/n^2)$-DP model learned by output perturbation will, with high probability, achieve a fairness level within this margin of that of the non-private model.

\subsection{Influence of the Number of Training Samples and Privacy Budget}
\label{sec:influence-n-epsilon}

We now verify numerically the rate at which fairness and accuracy levels decrease when increasing the number of training records or privacy budget. In \Cref{fig:fairness-fct-n-epsilon}, we plot the optimal model's equality of opportunity and accuracy, as a function of (i) in the first line, the number of samples $n$ used for training, or (ii) in the second line, the privacy budget $\epsilon$ (see \Cref{app:sec:additional-experiments} for results with other fairness measures). For each value of $n$ and $\epsilon$, we plot \Cref{thm:bound-on-fairness-private-models}'s theoretical guarantees (solid blue line). With $\epsilon=1$, our bounds give meaningful guarantees for $n \ge 100,000$ records on both \celebA and \folktables datasets (\Cref{fig:fairness-fct-n-f1,fig:fairness-fct-n-a1,fig:fairness-fct-n-f2,fig:fairness-fct-n-a2}).
When using all records,
we obtain meaningful bounds for
$\epsilon \ge 1$ for \celebA and $\epsilon \ge 0.1$ for \folktables
(\Cref{fig:fairness-fct-epsilon-f1,fig:fairness-fct-epsilon-a1,fig:fairness-fct-epsilon-f2,fig:fairness-fct-epsilon-a2}). Additionally, note that we obtain \emph{both upper and lower bounds} on fairness and accuracy, confirming remarks from \Cref{sec:bound-fairness}.

We also report the fairness and accuracy levels of $100$ private models computed by output perturbation (in green in \Cref{fig:fairness-fct-n-epsilon}). As predicted by our theory, their fairness and accuracy converges towards the ones of their non-private counterparts as $n$ and $\epsilon$ increase. Interestingly, our bounds seem to follow the same tendency as what we observe empirically (albeit with a larger multiplicative constant), suggesting that they capture the correct dependence in $n$ and $\epsilon$. We further discuss the tightness of our results in next section.

\subsection{Tightness of the Bound}
\label{sec:tightness-bound}

We now argue that the two major factors of looseness in our results are (i) the upper bound on $\norm{h^\priv-h^*}$ and (ii) the looseness of \Cref{assum:lipschitz-constant-fct-val}. While these cannot be improved in general, specific knowledge of $h^\priv$ and $h^*$ (that is typically not available due to privacy) can lead to tighter bounds. First, when the distance $\norm{h^\priv-h^*}$ is known, we can use its actual value rather than the upper bounds of \Cref{sec:two-priv-mech} (see dashed blue line in \Cref{fig:fairness-fct-n-epsilon}). Second, when both $h^\priv$ and $h^*$ are known, \Cref{assum:lipschitz-constant-fct-val} can be substantially refined (see details in \Cref{app:sec:refined-bounds}). We evaluate this bound for the private model that is the farthest away from the non-private one (see dotted blue line in \Cref{fig:fairness-fct-n-epsilon}). The resulting bound appears to be tight up to a small multiplicative constant. These two observations suggest that our bounds cannot be significantly tightened, unless one can obtain such knowledge through either private computation or additional assumptions on the data.

\section{Conclusion}
\label{sec:conclusion}

In this work, we proved that the fairness (and accuracy) costs induced by privacy in differentially private classification vanishes at a $\widetilde O(\sfrac{\sqrt{p}}{n})$ rate, where $n$ is the number of training records, and $p$ the number of parameters. This rate follows from a general statement on the regularity of a family of group fairness measures, that we prove to be pointwise Lipschitz with respect to the model. The pointwise Lipschitz constant explicitly depends on the confidence margin of the model, and may be different for each sensitive group. We also show it can be computed from a finite data sample. Importantly, our bounds do not require the knowledge of the optimal (non-private) model: they can thus be used in practical privacy-preserving scenarios. We numerically evaluate our bounds on real datasets, and highlight practical settings where non-trivial guarantees can be obtained.

While we illustrated our results for output perturbation and DP-SGD on strongly-convex problems, we stress that they are more general. Indeed, our \Cref{thm:bound-on-diff-fairness} holds for any pair of models. Consequently, it would apply to DP-SGD on non-convex problems, provided that we have a bound on the distance between the models obtained with and without privacy. Deriving a tight bound on this distance is a challenging problem, that constitutes an interesting future direction.

Finally, we stress that our results do not provide fairness guarantees \textit{per se}, but only bound the difference of fairness between models.
It is nonetheless a first step towards a more complete understanding of the interplay between privacy, fairness, and accuracy. We believe that our results can guide the design of fairer privacy-preserving machine learning algorithms. A first promising direction in this regard is to combine our bounds with fairness-promoting convex regularizers, as discussed in Remark~\ref{rmq:fairness-and-privacy-algorithm}. Another direction is the design of methods to privately learn models with large-margin guarantees, as recently considered by \citet{bassily2022differentially}. Our results, which explicitly depend on the confidence margin of the model, suggest that better fairness guarantees could be obtained for these methods.

\section*{Acknowledgements}

This work was supported by the Inria Exploratory Action FLAMED, by the Région Hauts de France (Projet STaRS Equité en apprentissage décentralisé respectueux de la vie privée), and by the French National Research
Agency (ANR) through the grants ANR-20-CE23-0015 (Project PRIDE), and ANR 22-PECY-0002 IPOP (Interdisciplinary Project on Privacy) project of the Cybersecurity PEPR.

\bibliography{references}
\bibliographystyle{icml2023}

\newpage
\appendix
\onecolumn

This appendix provides several examples of group fairness functions compatible with our framework (\Cref{app:sec:fairnessfunctions}), the proofs of the main theoretical results that were omitted in the main paper for the sake of readability (\Cref{sec-app:proof-thm-1,app:sec:proof-bound-on-diff-fairness,app:sec:proof-finitesample,sec-app:bound-output-pert,app:sec:convergence-dp-sgd}), and additional experiments (\Cref{app:sec:additional-experiments}).

\section{Fairness functions}
\label{app:sec:fairnessfunctions}

In this section we recall several well known fairness functions and show that they can be written in the form of \Cref{eq:fairness}.

\begin{myex}[\textbf{Equalized Odds \citep{hardt2016equality}}]
\label{ex:equalizedodds}
A model $h$ is fair for Equalized Odds when the probability of predicting the correct label is independent of the sensitive attribute, that is, $\forall (y,r) \in \cY \times \cS$
\begin{align*}
    F_{(y,r)}(h,D) = \prob\left(H(X) = Y \;\middle|\;Y = y, S = r\right) - \prob\left(H(X) = Y \;\middle|\;Y = y \right).
\end{align*}
We can then write $F_{(y,r)}(h,D)$ in the form of Equation~\eqref{eq:fairness} as
\begin{align}
    F_{(y,r)}(h, D) = C_{(y,r)}^{0} + \sum_{(y',r') \in \cY \times \cS} C_{(y,r)}^{(y',r')}\prob\left(H(x) = Y | Y = y', S = r' \right)
\end{align}
with
\begin{align*}
    C_{(y,r)}^{0} ={}& 0 \\
    C_{(y,r)}^{(y,r)} ={}& 1 - \prob\left(S = r \;\middle|\; Y = y\right) \\
    \text{  $\forall r' \neq r$, } C_{(y,r)}^{(y,r')} ={}& -\prob\left(S = r' \;\middle|\; Y = y\right)  \\
    \text{ $\forall y' \neq y$, $\forall r' \in \cS$, } C_{(y,r)}^{(y',r')} ={}& 0
\end{align*}
\end{myex}
\begin{proof}
We have that
\begin{align*}
    F_{(y,r)}(h,D) ={}& \prob\left(H(X) = Y \;\middle|\;Y = y, S = r\right) - \prob\left(H(X) = Y \;\middle|\;Y = y \right) \\
    ={}& \prob\left(H(X) = Y \;\middle|\;Y = y, S = r\right) - \sum_{r' \in \cS} \prob\left(H(X) = Y \;\middle|\;Y = y, S = r' \right)\prob\left(S = r' \;\middle|\;Y = y\right)
\end{align*}
which gives the result.
\end{proof}

\begin{myex}[\textbf{Equality of Opportunity \cite{hardt2016equality}}]
\label{ex:eopp}
A model $h$ is fair for Equality of Opportunity when the probability of predicting the correct label is independent of the sensitive attribute for the set of desirable outcomes $\cY' \subset \cY$, that is $\forall (y,r) \in \cY \times \cS$
\begin{align*}
    F_{(y,r)}(h,D) = \left\{ \begin{array}{ll}\prob\left(H(X) = Y \;\middle|\;Y = y, S = r\right) - \prob\left(H(X) = Y \;\middle|\;Y = y \right) & \text{if $y \in \cY'$,}\\
    0 & \text{otherwise.}\end{array} \right.
\end{align*}
We can then write $F_{(y,r)}(h,D)$ in the form of Equation~\eqref{eq:fairness} as
\begin{align}
    F_{(y,r)}(h, D) = C_{(y,r)}^{0} + \sum_{(y',r') \in \cY \times \cS} C_{(y,r)}^{(y',r')}\prob\left(H(X) = Y | Y = y', S = r' \right)
\end{align}
with, if $y \in \cY'$,
\begin{align*}
    C_{(y,r)}^{0} ={}& 0 \\
    C_{(y,r)}^{(y,r)} ={}& 1 - \prob\left(S = r \;\middle|\;Y = y\right) \\
    \text{$\forall r' \neq r$, } C_{(y,r)}^{(y,r')} ={}& -\prob\left(S = r' \;\middle|\;Y = y\right)  \\
    \text{$\forall y' \neq y$, $\forall r' \in \cS$, } C_{(y,r)}^{(y',r')} ={}& 0
\end{align*}
and, if $y \in \cY \setminus \cY'$,
\begin{align*}
    \text{$\forall y' \in \cY$, $\forall r' \in \cS$, } C_{(y,r)}^{(y',r')} = 0.
\end{align*}
\end{myex}

\begin{proof}
We consider the two cases. On the one hand, when $y \in \cY \setminus \cY'$, then we have that
\begin{align*}
    F_{(y,r)}(h,D) ={} 0
\end{align*}
which gives the first part of the result.
On the other hand, when $y \in \cY'$, then we have that
\begin{align*}
    F_{(y,r)}(h,D) ={}& \prob\left(H(X) = Y \;\middle|\;Y = y, S = r\right) - \prob\left(H(X) = Y \;\middle|\;Y = y \right) \\
    ={}& \prob\left(H(X) = Y \;\middle|\;Y = y, S = r\right) - \sum_{r' \in \cS} \prob\left(H(X) = Y \;\middle|\;Y = y, S = r' \right)\prob\left(S = r' \;\middle|\;Y = y\right)
\end{align*}
which gives the second part of the result.
\end{proof}

\begin{myex}[\textbf{Accuracy Parity \cite{zafar2017fairness}}]
\label{ex:accuracyparity}
A model $h$ is fair for Accuracy Parity when the probability of being correct is independent of the sensitive attribute, that is, $\forall (r) \in \cS$
\begin{align*}
    F_{(r)}(h,D) = \prob\left(H(X) = Y \;\middle|\;S = r\right) - \prob\left(H(X) = Y \right).
\end{align*}
We can then write $F_{(r)}(h,D)$ in the form of Equation~\eqref{eq:fairness} as
\begin{align}
    F_{(r)}(h, D) = C_{(r)}^{0} + \sum_{(r') \in \cS} C_{(r)}^{(r')}\prob\left(H(X) = Y | S = r'  \right)
\end{align}
with
\begin{align*}
    C_{(r)}^{0} ={}& 0\\
    C_{(r)}^{(r)} ={}& 1 - \prob\left(S = r\right) \\
    \text{  $\forall r' \neq r$, } C_{(r)}^{(r')} ={}& -\prob\left( S = r'\right)
\end{align*}
\end{myex}
\begin{proof}
We have that
\begin{align*}
    F_{(r)}(h,D) ={}& \prob\left(H(X) = Y \;\middle|\;S = r\right) - \prob\left(H(X) = Y \right) \\
    ={}& \prob\left(H(X) = Y \;\middle|\;S = r\right) - \sum_{r' \in \cS}\prob\left(H(X) = Y \;\middle|\;S = r'\right)\prob\left(S = r'\right)
\end{align*}
which gives the result.
\end{proof}

\begin{myex}[\textbf{Demographic Parity (Binary Labels) \cite{calders2009building}}]
A model $h$ is fair for Demographic Parity with binary labels when the probability of predicting a label is independent of the sensitive attribute, that is, $\forall (y,r) \in \cY \times \cS$
\begin{align*}
    F_{(y,r)}(h,D) = \prob\left(H(X) = y \;\middle|\;S = r\right) - \prob\left(H(X) = y \right).
\end{align*}
Assuming that given a label $y$, the second binary label is denoted $\bar{y}$, we can then write $F_{(y,r)}(h,D)$ in the form of Equation~\eqref{eq:fairness} as
\begin{align}
    F_{(y,r)}(h, D) = C_{(y,r)}^{0} + \sum_{(y',r') \in \cY \times \cS} C_{(y,r)}^{(y',r')}\prob\left(H(X) = Y | Y = y', S = r'  \right)
\end{align}
with
\begin{align*}
    C_{(y,r)}^{0} ={}& \prob\left(Y = y\right) - \prob\left(Y = y \;\middle|\;S = r\right) \\
    C_{(y,r)}^{(y,r)} ={}& \prob\left(Y = y \;\middle|\;S = r\right) - \prob\left(Y = y, S = r\right)\\
    C_{(y,r)}^{(\bar{y},r)} ={}& \prob\left(Y = \bar{y}, S = r\right) - \prob\left(Y = \bar{y} \;\middle|\;S = r\right) \\
    \text{  $\forall r' \neq r$, } C_{(y,r)}^{(y,r')} ={}& -\prob\left(Y = y, S = r'\right) \\
    \text{  $\forall r' \neq r$, } C_{(y,r)}^{(\bar{y},r')} ={}& \prob\left(Y = \bar{y}, S = r'\right)
\end{align*}
\end{myex}
\begin{proof}
We have that
\begin{align*}
    F_{(y,r)}(h,D) ={}& \prob\left(H(X) = y \;\middle|\;S = r\right) - \prob\left(H(X) = y \right) \\
     ={}& \prob\left(H(X) = y \;\middle|\;Y = y, S = r\right)\prob\left(Y = y \;\middle|\;S = r\right) + \prob\left(H(X) = y \;\middle|\;Y \neq y, S = r\right)\prob\left(Y \neq y \;\middle|\;S = r\right) \\
     &- \sum_{r' \in \cS} \big(\prob\left(H(X) = y \;\middle|\;Y = y, S = r' \right) \prob\left(Y = y, S = r'\right) \\[-1em]
     & \qquad\qquad + \prob\left(H(X) = y \;\middle|\;Y \neq y, S = r' \right) \prob\left(Y \neq y, S = r'\right) \big) \\
     ={}& \prob\left(H(X) = y \;\middle|\;Y = y, S = r\right)\prob\left(Y = y \;\middle|\;S = r\right)
     \\
     & \quad + 1 - \prob\left(H(X) \neq y \;\middle|\;Y \neq y, S = r\right)\prob\left(Y \neq y \;\middle|\;S = r\right) \\
     &- \sum_{r' \in \cS} \big(\prob\left(H(X) = y \;\middle|\;Y = y, S = r' \right) \prob\left(Y = y, S = r'\right) \\[-1em]
     & \qquad\qquad + 1 - \prob\left(H(X) \neq y \;\middle|\;Y \neq y, S = r' \right) \prob\left(Y \neq y, S = r'\right) \big).
\end{align*}
Here, we only consider binary labels, $y$ and $\bar{y}$. Hence, $H(X) \neq y \Leftrightarrow H(X) = \bar{y}$ and $Y \neq y \Leftrightarrow Y = \bar{y}$. Thus, we obtain
\begin{align*}
    F_{(y,r)}&(h,D)
    \\  ={}& \prob\left(H(X) = y \;\middle|\;Y = y, S = r\right)\prob\left(Y = y \;\middle|\;S = r\right) + \left(1 - \prob\left(H(X) = \bar{y} \;\middle|\;Y = \bar{y}, S = r\right)\right)\prob\left(Y = \bar{y} \;\middle|\;S = r\right) \\
     &- \sum_{r' \in \cS} \big(\prob\left(H(X) = y \;\middle|\;Y = y, S = r' \right) \prob\left(Y = y, S = r'\right) \\[-1em]
     & \qquad\qquad + \left(1 - \prob\left(H(X) = \bar{y} \;\middle|\;Y = \bar{y}, S = r' \right) \right)\prob\left(Y = \bar{y}, S = r'\right) \big) \\
     ={}& \prob\left(H(X) = y \;\middle|\;Y = y, S = r\right)\left[\prob\left(Y = y \;\middle|\;S = r\right) - \prob\left(Y = y, S = r\right)\right] \\
     &+ \prob\left(H(X) = \bar{y} \;\middle|\;Y = \bar{y}, S = r\right)\left[\prob\left(Y = \bar{y}, S = r\right) - \prob\left(Y = \bar{y} \;\middle|\;S = r\right)\right] \\
     &+ \sum_{r' \in \cS, r' \neq r} \prob\left(H(X) = y \;\middle|\;Y = y, S = r' \right)\left(-\prob\left(Y = y, S = r'\right)\right) \\
     &+ \sum_{r' \in \cS, r' \neq r} \prob\left(H(X) = \bar{y} \;\middle|\;Y = \bar{y}, S = r' \right) \prob\left(Y = \bar{y}, S = r'\right)\\
     &+ \prob\left(Y = \bar{y} \;\middle|\;S = r\right) - \prob\left(Y = \bar{y}\right) \\
     ={}& \prob\left(H(X) = Y \;\middle|\;Y = y, S = r\right)\left[\prob\left(Y = y \;\middle|\;S = r\right) - \prob\left(Y = y, S = r\right)\right] \\
     &+ \prob\left(H(X) = Y \;\middle|\;Y = \bar{y}, S = r\right)\left[\prob\left(Y = \bar{y}, S = r\right) - \prob\left(Y = \bar{y} \;\middle|\;S = r\right)\right] \\
     &+ \sum_{r' \in \cS, r' \neq r} \prob\left(H(X) = Y \;\middle|\;Y = y, S = r' \right)\left(-\prob\left(Y = y, S = r'\right)\right) \\
     &+ \sum_{r' \in \cS, r' \neq r} \prob\left(H(X) = Y \;\middle|\;Y = \bar{y}, S = r' \right) \prob\left(Y = \bar{y}, S = r'\right)\\
     &+ \prob\left(Y = y\right) - \prob\left(Y = y \;\middle|\;S = r\right)
\end{align*}
which gives the result.
\end{proof}

\section{Proof of \Cref{thm:bound-on-diff-proba}}
\label{sec-app:proof-thm-1}

\begin{theorem*}[Pointwise Lipschitzness of Conditional Negative Predictions]
    Let $\cH$ be a set of real vector-valued functions with $L_{X,Y}$ the Lipschitz constants defined in
  \Cref{assum:lipschitz-constant-fct-val}. Let $h, h' \in \cH$ be two models, $(X,Y,S)$ be a triple of random variables having distribution $\cD$, and $E$ be an arbitrary
  event. Assume that
  $\expect \left(\frac{L_{X,Y}}{\abs{\rho(h,X,Y)}} \;\middle|\; E \right) < +\infty$, then
  \begin{align*}
    \abs{ \prob( H(X) = Y \;\middle|\;E ) - \prob( H'(X) = Y \;\middle|\;E ) } \leq \expect \left(\frac{L_{X,Y}}{\abs{\rho(h,X,Y)}} \;\middle|\; E \right) \norm{h - h'}_{\cH}
      \enspace.
  \end{align*}
\end{theorem*}

\begin{proof}
The proof of this theorem is in two steps. First, we use the Lipschitz continuity property associated with $\cH$, the triangle inequality, and the union bound to show that $\abs{\prob\left(H(X) = Y \;\middle|\;E\right) - \prob\left(H'(X) = Y \;\middle|\; E\right)} \leq{} \prob\left(\frac{L_{X,Y}}{\abs{\rho(h,X,Y)}}  \leq \norm{h - h'}_\cH \;\middle|\; E\right)$. Then, applying Markov's inequality gives the desired result.

\paragraph{Bounding $\abs{\prob\left(H(X) = Y \;\middle|\; E\right) - \prob\left(H'(X) = Y \;\middle|\; E\right)}$.}
We have that
\begin{align*}
    \prob\left(H(X) = Y \;\middle|\; E\right) -& \prob\left(H'(X) = Y \;\middle|\; E\right) \\
    \leq{}& \prob\left(\rho(h,X,Y) \geq 0 \;\middle|\; E\right) - \prob\left(\rho(h',X,Y) > 0 \;\middle|\; E\right) \\
    ={}& \prob\left(\rho(h',X,Y) \leq 0 \;\middle|\; E\right) - \prob\left(\rho(h,X,Y) < 0 \;\middle|\; E\right) \\
    ={}& \prob\left(\rho(h',X,Y) - \rho(h,X,Y) + \rho(h,X,Y) \leq 0 \;\middle|\; E\right) - \prob\left(\rho(h,X,Y) < 0 \;\middle|\; E\right) \\
    ={}& \prob\left(\rho(h,X,Y) \leq \rho(h,X,Y) - \rho(h',X,Y) \;\middle|\; E\right) - \prob\left(\rho(h,X,Y) < 0 \;\middle|\; E\right) \\
    \leq{}& \prob\left(\rho(h,X,Y) \leq \abs{\rho(h,X,Y) - \rho(h',X,Y)} \;\middle|\; E\right) - \prob\left(\rho(h,X,Y) < 0 \;\middle|\; E\right) \\
    \cause{\Cref{assum:lipschitz-constant-fct-val}.}
    \leq{}& \prob\left(\rho(h,X,Y) \leq L_{X,Y} \norm{h - h'}_\cH \;\middle|\; E\right) - \prob\left(\rho(h,X,Y) < 0 \;\middle|\; E\right) \\
    ={}& \prob\left(\rho(h,X,Y) < 0 \bigcup 0 \leq \rho(h,X,Y) \leq L_{X,Y} \norm{h - h'}_\cH \;\middle|\; E\right) - \prob\left(\rho(h,X,Y) < 0 \;\middle|\; E\right) \\
    \cause{Union bound on disjoint events.}
    ={}& \prob\left(0 \leq \rho(h,X,Y) \leq L_{X,Y} \norm{h - h'}_\cH \;\middle|\; E\right)\\
    \leq{}& \prob\left(\abs{\rho(h,X,Y)} \leq L_{X,Y} \norm{h - h'}_\cH \;\middle|\; E\right)\\
    ={}& \prob\left(\frac{\abs{\rho(h,X,Y)}}{L_{X,Y}} \leq \norm{h - h'}_\cH \;\middle|\; E\right)
\end{align*}
Similarly, we have that
\begin{align*}
    \prob\left(H'(X) = Y \;\middle|\; E\right) -& \prob\left(H(X) = Y \;\middle|\; E\right) \\
    \leq{}& \prob\left(\rho(h',X,Y) \geq 0 \;\middle|\; E\right) - \prob\left(\rho(h,X,Y) > 0 \;\middle|\; E\right) \\
    ={}& \prob\left(\rho(h',X,Y) + \rho(h,X,Y) - \rho(h,X,Y) \geq 0 \;\middle|\; E\right) - \prob\left(\rho(h,X,Y) > 0 \;\middle|\; E\right) \\
    ={}& \prob\left(\rho(h,X,Y) \geq - \left(\rho(h',X,Y) - \rho(h,X,Y)\right) \;\middle|\; E\right) - \prob\left(\rho(h,X,Y) > 0 \;\middle|\; E\right) \\
    \leq{}& \prob\left(\rho(h,X,Y) \geq - \abs{\rho(h',X,Y) - \rho(h,X,Y)} \;\middle|\; E\right) - \prob\left(\rho(h,X,Y) > 0 \;\middle|\; E\right) \\
    \cause{\Cref{assum:lipschitz-constant-fct-val}}
    \leq{}& \prob\left(\rho(h,X,Y) \geq - L_{X,Y} \norm{h - h'}_\cH \;\middle|\; E\right) - \prob\left(\rho(h,X,Y) > 0 \;\middle|\; E\right) \\
    ={}& \prob\left(\rho(h,X,Y) > 0 \bigcup 0 \geq \rho(h,X,Y) \geq - L_{X,Y} \norm{h - h'}_\cH \;\middle|\; E\right) - \prob\left(\rho(h,X,Y) > 0 \;\middle|\; E\right) \\
    \cause{Union bound on disjoint events.}
    ={}& \prob\left(0 \geq \rho(h,X,Y) \geq - L_{X,Y} \norm{h - h'}_\cH \;\middle|\; E\right) \\
    \leq{}& \prob\left(- \abs{\rho(h,X,Y)} \geq - L_{X,Y} \norm{h - h'}_\cH \;\middle|\; E\right) \\
    ={}& \prob\left(\frac{\abs{\rho(h,X,Y)}}{L_{X,Y}} \leq \norm{h - h'}_\cH \;\middle|\; E\right)
\end{align*}
It implies that
\begin{align*}
    \abs{\prob\left(H(X) = Y \;\middle|\; E\right) - \prob\left(H'(X) = Y \;\middle|\; E\right)} \leq{} \prob\left(\frac{\abs{\rho(h,X,Y)}}{L_{X,Y}} \leq \norm{h - h'}_\cH \;\middle|\; E\right)
\end{align*}

\paragraph{Bounding $\prob\left(\frac{\abs{\rho(h,X,Y)}}{L_{X,Y}} \leq \norm{h - h'}_\cH \;\middle|\; E\right)$.}
We use the Markov's Inequality and we assume that $\expect \left(\frac{L_{X,Y}}{\abs{\rho(h,X,Y)}} \;\middle|\; E \right) < +\infty$. Hence, we have that
\begin{align*}
    \prob\left(\frac{\abs{\rho(h,X,Y)}}{L_{X,Y}} \leq \norm{h - h'}_\cH \;\middle|\; E\right) ={}& \prob\left(\frac{L_{X,Y}}{\abs{\rho(h,X,Y)}} \geq \frac{1}{\norm{h - h'}_\cH} \;\middle|\; E\right) \\
    \cause{Markov's inequality.}
    \leq{}& \expect\left(\frac{L_{X,Y}}{\abs{\rho(h,X,Y)}}\;\middle|\; E\right)\norm{h - h'}_\cH
\end{align*}
It concludes the proof.
\end{proof}

\begin{remark}
    \label{app:rem:chernoff-instead-markov}
In the last step of the proof of \Cref{thm:bound-on-diff-proba}, we can also use the Chernoff bound:
\begin{align*}
    \prob\left(\frac{\abs{\rho(h,X,Y)}}{L_{X,Y}} \leq \norm{h - h'}_\cH \;\middle|\; E\right) ={}& \prob\left(\exp\left(-t\frac{\abs{\rho(h,X,Y)}}{L_{X,Y}}\right) \geq \exp\left(-t\norm{h - h'}_\cH\right) \;\middle|\; E\right) \\
    \leq{}& \expect\left( \exp\left(-t\frac{\abs{\rho(h,X,Y)}}{L_{X,Y}}\right) \;\middle|\; E\right)\exp\left(t\norm{h - h'}_\cH\right)\enspace.
\end{align*}
A correct choice of $t$ would lead to potentially tighter bounds than the Markov's inequality at the expense of readability.
\end{remark}

\begin{remark}
    \label{app:rem:large-enough-margins}
    Before using Markov's inequality or Chernoff bound in \Cref{thm:bound-on-diff-proba}, we can modify the probability as
    \begin{align*}
    \prob\left(\frac{\abs{\rho(h,X,Y)}}{L_{X,Y}} \leq \norm{h - h'}_\cH \;\middle|\; E\right)
    ={}&
    \prob\left(\left[\frac{\abs{\rho(h,X,Y)}}{L_{X,Y}}\right]^{\norm{h - h'}_\cH} \leq \norm{h - h'}_\cH \;\middle|\; E\right)
    \enspace,
    \end{align*}
where
    \begin{align*}
    \left[\frac{\abs{\rho(h,X,Y)}}{L_{X,Y}}\right]^{\norm{h - h'}_\cH}
    & =
    \begin{cases}
    \frac{\abs{\rho(h,X,Y)}}{L_{X,Y}} & \text{if~~} \abs{\rho(h,X,Y)} \le L_{X,Y}\norm{h - h'}_\cH \enspace, \\
    +\infty & \text{otherwise} \enspace.
    \end{cases}
    \end{align*}
    This essentially means that, whenever the model's margin on a data record is large enough, its precise value is no more meaningful, as its prediction will not change whatsoever. The remaining of \Cref{thm:bound-on-diff-proba}'s proof is unchanged, except with~$\left[\frac{\abs{\rho(h,X,Y)}}{L_{X,Y}}\right]^{\norm{h - h'}_\cH}$ instead of $\frac{\abs{\rho(h,X,Y)}}{L_{X,Y}}$.

    Note that this can lead to much tighter bounds. Notably, when distance $\norm{h-h'}_{\cH}$ between $h$ and $h'$ is small enough, the difference of fairness may even become zero.
\end{remark}

\section{Proof of \Cref{thm:bound-on-diff-fairness}}
\label{app:sec:proof-bound-on-diff-fairness}

\begin{theorem*}[Pointwise Lipschitzness of Fairness]
  Let $h, h' \in \cH$, $L_{X,Y}$ be defined as in
  \Cref{assum:lipschitz-constant-fct-val}, and $(X, S, Y) \sim \cD$. For any fairness notion of the form of Equation~\eqref{eq:fairness}, we have:
  \begin{align*}
      \forall k \in [K], \abs{F_k(h,D) - F_k(h',D)} \leq \chi_k(h, D) \norm{h - h'}_{\cH}
      \enspace,
  \end{align*}
  with $ \chi_k(h, D) = \sum_{k'=1}^K \abs{C_k^{k'}}\expect \left( \frac{1}{\abs{h(X)}} \;\middle|\; D_{k'} \right)$.
  Similarly, for the aggregate measure of fairness defined in Equation~\eqref{eq:aggregatefairness}, we have:
  \begin{align*}
      \abs{\text{Fair}(h,D) - \text{Fair}(h',D)} \leq \frac{1}{K} \sum_{k=1}^K \chi_k(h, D) \norm{h - h'}_{\cH}
      \enspace.
  \end{align*}
\end{theorem*}
\begin{proof}
The first part follows from the following derivation. For all $k$,
\begin{align*}
\abs{F_k(h,D) - F_k(h',D)} ={}& \abs{ C_k^{0} + \sum_{k'=1}^K C_k^{k'}\prob\left(H(X) = Y | D_{k'}\right) - C_k^{0} - \sum_{k'=1}^K C_k^{k'}\prob\left(H'(X) = Y | D_{k'}\right) } \\
={}& \abs{ \sum_{k'=1}^K C_k^{k'}\Big(\prob\left(H(X) = Y | D_{k'}\right) - \prob\left(H'(X) = Y | D_{k'}\right) \Big) } \\
\cause{Triangle inequality.}
\leq{}& \sum_{k'=1}^K \abs{C_k^{k'}}\abs{ \prob\left(H(X) = Y | D_{k'}\right) - \prob\left(H'(X) = Y | D_{k'}\right) } \\
\cause{Theorem~\ref{thm:bound-on-diff-proba}.}
\leq{}& \sum_{k'=1}^K \abs{C_k^{k'}}\expect\left(\frac{L_{X,Y}}{\abs{\rho(h,X,Y)}}\;\middle|\; D_{k'}\right)\norm{h - h'}_\cH \enspace.
\end{align*}

The second part is obtained thanks to the triangle inequality:
\begin{align*}
    \abs{\text{Fair}(h,D) - \text{Fair}(h',D)} ={}& \abs{\frac{1}{K} \sum_{k=1}^K \abs{F_k(h, D)} - \frac{1}{K} \sum_{k=1}^K \abs{F_k(h', D)}} \\
    \cause{Triangle inequality.}
    \leq{}& \frac{1}{K} \sum_{k=1}^K \abs{\abs{F_k(h, D)} - \abs{F_k(h', D)}} \\
    \cause{Reverse triangle inequality.}
    \leq{}& \frac{1}{K} \sum_{k=1}^K \abs{F_k(h, D) - F_k(h', D)}\enspace,
\end{align*}
which gives the claim when combined with the first part of the theorem.
\end{proof}

\section{Proof of \Cref{lem:finitesample}}
\label{app:sec:proof-finitesample}
\begin{lemma*}[Finite Sample analysis]
Let $D$ be a finite sample of $n \geq \frac{8\log\left(\frac{2K+1}{\delta}\right)}{\min_{k'}p_{k'}}$ examples drawn i.i.d. from $\cD$, where $p_{k'}$ is the true proportion of examples from group $k'$. Assume that $\prob_{D \sim \cD^n}\big(\sum_{k'=0}^K\babs{C_k^{k'} - \widehat{C}_k^{k'}} > \alpha_{C}\big) \leq B_3\exp\left(-B_4\alpha_{C}^2n\right)$. Let $\cH$ be an hypothesis space and $d_{\cH}$ be the Natarajan dimension of $\cH$.
\begin{itemize}
    \item \textbf{Assuming that $h$ and $h'$ are independent of $D$.} With probability $1-\delta$ over the choice of $D$,
    \begin{align*}
        \abs{F_k(h,D) - \widehat{F}_k(h',D)} \leq{} \widehat{\chi}_k(h, D)\norm{h - h'}_{\cH} + \widetilde{O}\left(\sum_{k'=1}^K \abs{\widehat{C}_k^{k'}}\sqrt{\frac{\log\left(\sfrac{K}{\delta}\right)}{np_{k'}}}\right)
    \end{align*}
    \item \textbf{Assuming that $h$ and $h'$ are dependent of $D$.} With probability $1-\delta$ over the choice of $D$, $\forall h,h' \in \cH$,
    \begin{align*}
        \abs{F_k(h,D) - \widehat{F}_k(h',D)} \leq \widehat{\chi}_k(h, D)\norm{h - h'}_{\cH}
        + \widetilde{O}\left(\sum_{k'=1}^K \abs{\widehat{C}_k^{k'}}\sqrt{\frac{d_{\cH} + \log\left(\sfrac{K}{\delta}\right)}{np_{k'}}}\right)
    \end{align*}
\end{itemize}
\end{lemma*}

\begin{proof}
First of all, notice that we have
\begin{align*}
    \abs{F_k(h,D) - \widehat{F}_k(h',D)} ={}& \abs{F_k(h,D) - \widehat{F}_k(h,D) + \widehat{F}_k(h,D) - \widehat{F}_k(h',D)} \\
    \leq{}& \abs{F_k(h,D) - \widehat{F}_k(h,D)} + \abs{\widehat{F}_k(h,D) - \widehat{F}_k(h',D)} \\
    \cause{\Cref{thm:bound-on-diff-fairness}}
    \leq{}& \abs{F_k(h,D) - \widehat{F}_k(h,D)} + \widehat{\chi}_k(h', D)\norm{h - h'}_{\cH}
\end{align*}

Hence it remains to bound the first term. By definition of our fairness notions, notice that we have the following.
\begin{align*}
    \abs{F_k(h,D) - \widehat{F}_k(h,D)} ={}& \abs{ C_k^{0} + \sum_{k'=1}^K C_k^{k'}\prob\left(H(X) = Y | D_{k'}\right) - \widehat{C}_k^{0} - \sum_{k'=1}^K \widehat{C}_k^{k'}\widehat{\prob}\left(H(X) = Y | D_{k'}\right) } \\
    \leq{}& \abs{ C_k^{0}- \widehat{C}_k^{0}} + \abs{\sum_{k'=1}^K C_k^{k'}\prob\left(H(X) = Y | D_{k'}\right) - \widehat{C}_k^{k'}\widehat{\prob}\left(H(X) = Y | D_{k'}\right) } \\
    \leq{}& \abs{ C_k^{0}- \widehat{C}_k^{0}} + \sum_{k'=1}^K \left|C_k^{k'}\prob\left(H(X) = Y | D_{k'}\right) - \widehat{C}_k^{k'}\prob\left(H(X) = Y | D_{k'}\right)\right.\\
    &+ \left.  \widehat{C}_k^{k'}\prob\left(H(X) = Y | D_{k'}\right) - \widehat{C}_k^{k'}\widehat{\prob}\left(H(X) = Y | D_{k'}\right) \right| \\
    \leq{}& \abs{ C_k^{0}- \widehat{C}_k^{0}} + \sum_{k'=1}^K \abs{C_k^{k'}\prob\left(H(X) = Y | D_{k'}\right) - \widehat{C}_k^{k'}\prob\left(H(X) = Y | D_{k'}\right)} \\
    &+ \sum_{k'=1}^K\abs{\widehat{C}_k^{k'}\prob\left(H(X) = Y | D_{k'}\right) - \widehat{C}_k^{k'}\widehat{\prob}\left(H(X) = Y | D_{k'}\right) } \\
    \leq{}& \abs{ C_k^{0}- \widehat{C}_k^{0}} + \sum_{k'=1}^K \abs{\widehat{C}_k^{k'}}\abs{\prob\left(H(X) = Y | D_{k'}\right) - \widehat{\prob}\left(H(X) = Y | D_{k'}\right)} \\
    &+ \sum_{k'=1}^K\abs{C_k^{k'} - \widehat{C}_k^{k'}}\prob\left(H(X) = Y | D_{k'}\right) \\
    \leq{}& \abs{ C_k^{0}- \widehat{C}_k^{0}} + \sum_{k'=1}^K \abs{\widehat{C}_k^{k'}}\abs{\prob\left(H(X) = Y | D_{k'}\right) - \widehat{\prob}\left(H(X) = Y | D_{k'}\right)} \\
    &+ \sum_{k'=1}^K\abs{C_k^{k'} - \widehat{C}_k^{k'}} \\
    \leq{}& \sum_{k'=0}^K\abs{C_k^{k'} - \widehat{C}_k^{k'}} + \sum_{k'=1}^K \abs{\widehat{C}_k^{k'}}\abs{\prob\left(H(X) = Y | D_{k'}\right) - \widehat{\prob}\left(H(X) = Y | D_{k'}\right)}
\end{align*}

We now need to consider two cases, depending on whether $h$ depends on $D$ or not.

\paragraph{Assuming that $h$ is independent of $D$.} In this case, our goal is to upper bound
\begin{align*}
    \prob_{D \sim \cD^n}\left(\abs{F_k(h,D) - \widehat{F}_k(h,D)} > \alpha_{C} + \sum_{k'=1}^K \abs{\widehat{C}_k^{k'}}\alpha_{k'}\right)
\end{align*}
Notice that, using the same trick that \citet{woodworth2017learning} used to prove their Equation~(38), we have that
\begin{align*}
    \prob_{D \sim \cD^n}& \left(\abs{F_k(h,D) - \widehat{F}_k(h,D)} > \alpha_{C} + \sum_{k'=1}^K \abs{\widehat{C}_k^{k'}}\alpha_{k'}\right) %
    \\
    \leq{}& \prob_{D \sim \cD^n}\left(\sum_{k'=0}^K\abs{C_k^{k'} - \widehat{C}_k^{k'}} + \sum_{k'=1}^K \abs{\widehat{C}_k^{k'}}\abs{\prob\left(H(X) = Y | D_{k'}\right) - \widehat{\prob}\left(H(X) = Y | D_{k'}\right)} > \alpha_{C} + \sum_{k'=1}^K \abs{\widehat{C}_k^{k'}}\alpha_{k'}\right) \\
    \leq{}& \prob_{D \sim \cD^n}\left(\sum_{k'=0}^K\abs{C_k^{k'} - \widehat{C}_k^{k'}} > \alpha_{C} \bigcup \left[\bigcup_{k'=1}^K \abs{\prob\left(H(X) = Y | D_{k'}\right) - \widehat{\prob}\left(H(X) = Y | D_{k'}\right)} > \alpha_{k'}\right]\right) \\
    \leq{}& \prob_{D \sim \cD^n}\left(\sum_{k'=0}^K\abs{C_k^{k'} - \widehat{C}_k^{k'}} > \alpha_{C}\right) +  \sum_{k'=1}^K\prob_{D \sim \cD^n}\left( \abs{\prob\left(H(X) = Y | D_{k'}\right) - \widehat{\prob}\left(H(X) = Y | D_{k'}\right)} > \alpha_{k'}\right) \\
    \leq{}& \prob_{D \sim \cD^n}\left(\sum_{k'=0}^K\abs{C_k^{k'} - \widehat{C}_k^{k'}} > \alpha_{C}\right) \\
    &+  \sum_{k'=1}^K\sum_{i=0}^n\prob_{D \sim \cD^n}\left( \abs{\prob\left(H(X) = Y | D_{k'}\right) - \widehat{\prob}\left(H(X) = Y | D_{k'}\right)} > \alpha_{k'} \;\middle|\; \abs{D_{k'}} = i\right)\prob_{D \sim \cD^n}\left(\abs{D_{k'}} = i\right) \\
    \cause{Let $p_{k'} = \prob\left(D_{k'}\right)$}
    \leq{}& \prob_{D \sim \cD^n}\left(\sum_{k'=0}^K\abs{C_k^{k'} - \widehat{C}_k^{k'}} > \alpha_{C}\right) \\
    &+  \sum_{k'=1}^K\sum_{i=0}^{\sfrac{np_{k'}}{2}-1}\prob_{D \sim \cD^n}\left( \abs{\prob\left(H(X) = Y | D_{k'}\right) - \widehat{\prob}\left(H(X) = Y | D_{k'}\right)} > \alpha_{k'} \;\middle|\; \abs{D_{k'}} = i\right)\prob_{D \sim \cD^n}\left(\abs{D_{k'}} = i\right) \\
    &+  \sum_{k'=1}^K\sum_{i=\sfrac{np_{k'}}{2}}^{n}\prob_{D \sim \cD^n}\left( \abs{\prob\left(H(X) = Y | D_{k'}\right) - \widehat{\prob}\left(H(X) = Y | D_{k'}\right)} > \alpha_{k'} \;\middle|\; \abs{D_{k'}} = i\right)\prob_{D \sim \cD^n}\left(\abs{D_{k'}} = i\right) \\
    \leq{}& \prob_{D \sim \cD^n}\left(\sum_{k'=0}^K\abs{C_k^{k'} - \widehat{C}_k^{k'}} > \alpha_{C}\right) \\
    &+  \sum_{k'=1}^K\sum_{i=0}^{\sfrac{np_{k'}}{2}-1}\prob_{D \sim \cD^n}\left(\abs{D_{k'}} = i\right) \\
    &+  \sum_{k'=1}^K\sum_{i=\sfrac{np_{k'}}{2}}^{n}\prob_{D_{k'} \sim \cD_{k'}^i}\left( \abs{\prob\left(H(X) = Y | D_{k'}\right) - \widehat{\prob}\left(H(X) = Y | D_{k'}\right)} > \alpha_{k'}\right)\prob_{D \sim \cD^n}\left(\abs{D_{k'}} = i\right) \\
    \leq{}& \prob_{D \sim \cD^n}\left(\sum_{k'=0}^K\abs{C_k^{k'} - \widehat{C}_k^{k'}} > \alpha_{C}\right) \\
    &+  \sum_{k'=1}^K \prob_{D \sim \cD^n}\left(\abs{D_{k'}} < \frac{np_{k'}}{2}\right) \\
    &+  \sum_{k'=1}^K\sum_{i=\sfrac{np_{k'}}{2}}^{n}\prob_{D_{k'} \sim \cD_{k'}^i}\left( \abs{\prob\left(H(X) = Y | D_{k'}\right) - \widehat{\prob}\left(H(X) = Y | D_{k'}\right)} > \alpha_{k'}\right)\prob_{D \sim \cD^n}\left(\abs{D_{k'}} = i\right)
\end{align*}

Using Hoeffding's inequality, we can show that
\begin{align*}
    \prob_{D_{k'} \sim \cD_{k'}^{n_{k'}}}\left( \abs{\prob\left(H(X) = Y | D_{k'}\right) - \widehat{\prob}\left(H(X) = Y | D_{k'}\right)} > \beta\right) \leq 2\exp\left(-2\beta^2 n_{k'}\right)
\end{align*}
which implies
\begin{align*}
    \prob_{D \sim \cD^n}\left(\abs{F_k(h,D) - \widehat{F}_k(h,D)} > \alpha_{C} + \sum_{k'=1}^K \abs{\widehat{C}_k^{k'}}\alpha_{k'}\right) \negspace{15em}& \\
    \cause{$2\exp\left(-2\beta^2 i\right) \geq 2\exp\left(-2\beta^2 (i+1)\right)$.}
    \leq{}& \prob_{D \sim \cD^n}\left(\sum_{k'=0}^K\abs{C_k^{k'} - \widehat{C}_k^{k'}} > \alpha_{C}\right) \\
    &+  \sum_{k'=1}^K \prob_{D \sim \cD^n}\left(\abs{D_{k'}} < \frac{np_{k'}}{2}\right) \\
    &+  \sum_{k'=1}^K 2\exp\left(-\alpha_{k'}^2np_{k'}\right) \\
    \cause{By assumption, $\prob_{D \sim \cD^n}\left(\sum_{k'=0}^K\abs{C_k^{k'} - \widehat{C}_k^{k'}} > \alpha_{C}\right) \leq B_3\exp\left(-B_4\alpha_{C}^2n\right)$.}
    \leq{}& B_3\exp\left(-B_4\alpha_{C}^2n\right) \\
    &+  \sum_{k'=1}^K \prob_{D \sim \cD^n}\left(\abs{D_{k'}} < \frac{np_{k'}}{2}\right) \\
    &+  \sum_{k'=1}^K 2\exp\left(-\alpha_{k'}^2np_{k'}\right) \\
    \cause{Chernoff multiplicative bound.}
     \leq{}& B_3\exp\left(-B_4\alpha_{C}^2n\right) \\
    &+  \sum_{k'=1}^K \exp\left(-\frac{np_{k'}}{8}\right) \\
    &+  \sum_{k'=1}^K 2\exp\left(-\alpha_{k'}^2np_{k'}\right)
\end{align*}

Now, by assumption that $n \geq \frac{8\log\left(\frac{2K+1}{\delta}\right)}{\min_{k'}p_{k'}}$ and setting
\begin{align*}
    \alpha_{C} ={}& \sqrt{\frac{\log\left(\frac{B_3(2K+1)}{\delta}\right)}{B_4n}} \\
    \alpha_{k'} ={}& \sqrt{\frac{\log\left(\frac{2(2K+1)}{\delta}\right)}{np_{k'}}}
\end{align*}
yields that, with probability at least $1-\delta$,
\begin{align*}
    \abs{F_k(h,D) - \widehat{F}_k(h,D)} \leq \sqrt{\frac{\log\left(\frac{B_3(2K+1)}{\delta}\right)}{B_4n}} + \sum_{k'=1}^K \abs{\widehat{C}_k^{k'}} \sqrt{\frac{\log\left(\frac{2(2K+1)}{\delta}\right)}{np_{k'}}}
\end{align*}

\paragraph{Assuming that $h$ is dependent of $D$.}
In this case, our goal is to bound
\begin{align*}
    \prob_{D \sim \cD^n}\left(\sup_{h \in \cH} \abs{F_k(h,D) - \widehat{F}_k(h,D)} > \alpha_{C} + \sum_{k'=1}^K \abs{\widehat{C}_k^{k'}}\alpha_{k'}\right)
\end{align*}
Using similar arguments that in the independent case, we have that
\begin{align*}
    \prob_{D \sim \cD^n}&\left(\sup_{h \in \cH} \abs{F_k(h,D) - \widehat{F}_k(h,D)} > \alpha_{C} + \sum_{k'=1}^K \abs{\widehat{C}_k^{k'}}\alpha_{k'}\right) %
    \\
    \leq{}& \prob_{D \sim \cD^n}\left(\sum_{k'=0}^K\abs{C_k^{k'} - \widehat{C}_k^{k'}} > \alpha_{C}\right) \\
    &+  \sum_{k'=1}^K \prob_{D \sim \cD^n}\left(\abs{D_{k'}} < \frac{np_{k'}}{2}\right) \\
    &+  \sum_{k'=1}^K\sum_{i=\sfrac{np_{k'}}{2}}^{n}\prob_{D_{k'} \sim \cD_{k'}^i}\left(\sup_{h \in \cH} \abs{\prob\left(H(X) = Y | D_{k'}\right) - \widehat{\prob}\left(H(X) = Y | D_{k'}\right)} > \alpha_{k'}\right)\prob_{D \sim \cD^n}\left(\abs{D_{k'}} = i\right)
\end{align*}

Using the Multiclass Fundamental Theorem \citep[Theorem 29.3, Lemma 29.4]{shalev2014understanding} with $d_{\cH}$ the Natarajan dimension of $\cH$, we have that
\begin{align*}
    \prob_{D_{k'} \sim \cD_{k'}^{n_{k'}}}\left( \sup_{h \in \cH} \abs{\prob\left(H(X) = Y | D_{k'}\right) - \widehat{\prob}\left(H(X) = Y | D_{k'}\right)} > \beta\right) \leq 8 n_{k'}^{d_{\cH}}\abs{\cY}^{2d_{\cH}} \exp\left(-\frac{n_{k'}\beta^2}{32}\right)
\end{align*}
which implies
\begin{align*}
    \prob_{D \sim \cD^n}\left(\sup_{h \in \cH} \abs{F_k(h,D) - \widehat{F}_k(h,D)} > \alpha_{C} + \sum_{k'=1}^K \abs{\widehat{C}_k^{k'}}\alpha_{k'}\right) \negspace{15em}& \\
    \leq{}& \prob_{D \sim \cD^n}\left(\sum_{k'=0}^K\abs{C_k^{k'} - \widehat{C}_k^{k'}} > \alpha_{C}\right) \\
    &+  \sum_{k'=1}^K \prob_{D \sim \cD^n}\left(\abs{D_{k'}} < \frac{np_{k'}}{2}\right) \\
    &+  \sum_{k'=1}^K 8 \left(\frac{np_{k'}}{2}\right)^{d_{\cH}}\abs{\cY}^{2d_{\cH}} \exp\left(-\frac{\left(\frac{np_{k'}}{2}\right)\alpha_{k'}^2}{32}\right)
\end{align*}

Now, by assumption that $n \geq \frac{8\log\left(\frac{2K+1}{\delta}\right)}{\min_{k'}p_{k'}}$ and setting
\begin{align*}
    \alpha_{C} ={}& \sqrt{\frac{\log\left(\frac{B_3(2K+1)}{\delta}\right)}{B_4n}} \\
    \alpha_{k'} ={}& \sqrt{\frac{64\log\left(\frac{8 \left(\frac{np_{k'}}{2}\right)^{d_{\cH}}\abs{\cY}^{2d_{\cH}}(2K+1)}{\delta}\right)}{np_{k'}}} = \sqrt{\frac{64\left(d_{\cH}\left(\log\left( \frac{np_{k'}}{2}\right) + 2\log\left( \abs{\cY}\right)\right) +\log\left(\frac{8(2K+1)}{\delta}\right)\right)}{np_{k'}}}
\end{align*}
yields that, with probability at least $1-\delta$,  $\forall h \in \cH$
\begin{align*}
    \abs{F_k(h,D) - \widehat{F}_k(h,D)} \leq \sqrt{\frac{\log\left(\frac{B_3(2K+1)}{\delta}\right)}{B_4n}} + \sum_{k'=1}^K \abs{\widehat{C}_k^{k'}}\sqrt{\frac{64\left(d_{\cH}\left(\log\left( \frac{np_{k'}}{2}\right) + 2\log\left( \abs{\cY}\right)\right) +\log\left(\frac{8(2K+1)}{\delta}\right)\right)}{np_{k'}}} \enspace.
\end{align*}
This concludes the proof.
\end{proof}

\section{Bound for Output Perturbation (Proof of \Cref{lemma:sensitivity-output-perturbation})}
\label{sec-app:bound-output-pert}

\begin{lemma*}
  Let $h^{\priv}$ be the vector released by output
  perturbation with noise
  $\sigma^2 = \sfrac{8\Lambda^2\log(1.25/\delta)}{\mu^2 n^2 \epsilon^2}$,
  and $0 < \zeta < 1$, then with probability at least $1 - \zeta$,
  \begin{align*}
    \norm{h^{\priv} - h^*}_2^2
      \le \frac{32p\Lambda^2\log(1.25/\delta)\log(2/\zeta)}{\mu^2 n^2 \epsilon^2}
      \enspace.
  \end{align*}
\end{lemma*}

\begin{proof}

We prove this lemma in two steps. First, we show that for a given sensitivity, the distance $\norm{h^\priv - h^*}$ is bounded. Second, we estimate the sensitivity.

\textbf{Bounding the Error.}
Let $\Delta$ be the sensitivity of the function $D \rightarrow \argmin_{w\in\cC} f(w; D)$. Its value can be released under $(\epsilon, \delta)$ differential privacy \citep{chaudhuri2011Differentially,lowy2021output} as follows:
\begin{align}
    h^{\priv}
    & = h^* + \cN(0, \sigma^2 \mathbb I_p) \enspace,
\end{align}
where $\sigma^2 = \frac{2 \Delta^2 \log(1.25/\delta)}{\epsilon^2}$ and $h^* = \argmin_{h \in \cC} f(h)$. Then, Chernoff's bound gives, for $t, \alpha > 0$,
\begin{align}
    \prob( \norm{h^{priv} - h^*}^2 \ge \alpha )
    & \le \exp(-t\alpha) \expect (\exp(t \norm{h^{priv} - h^*}^2)) \\
    & = \exp(-t\alpha) \prod_{j=1}^p \expect (\exp(t (h_j^{priv} - h_j^*)^2))
    \enspace,
\end{align}
by independence of the noise's $p$ coordinates. Since $h_j^{priv} - h_j^*$ is a Gaussian random variable of mean $0$ and variance $\sigma^2$, we can compute $\expect (\exp(t(h_j^{priv} - h_j^*))) = (1-2t\sigma^2)^{-1/2}$. We then obtain
\begin{align}
    \prob( \norm{h^{priv} - h^*}^2 \ge \alpha )
    & \le \exp(-t\alpha) (1 - 2t\sigma^2)^{-p/2}
    \enspace.
\end{align}
Let $t=1/4p\sigma^2$, then it holds that $1 - 2t\sigma^2 = 1 - 1/2p \le 1$ and
\begin{align}
    (1 - 2t\sigma^2)^{-p/2}
    & = \exp\left(-\frac{p}{2} \log(1 - \frac{1}{2p})\right)
      \le \exp\left(\frac{1}{2(1 - \frac{1}{p})}\right)
      \le \exp(1/2)
      \le 2
      \enspace,
\end{align}
since $\frac{p}{2} \log(1-\frac{1}{2p}) \ge \frac{p}{2} \frac{-1/2p}{1-1/2p} \ge -\frac{1}{2} $. Let $0 < \zeta < 1$, $t = 1/4p\sigma^2$ and $\alpha = 4p\sigma^2 \log(2/\zeta)$, we have proven
\begin{align}
    \prob( \norm{h^{priv} - h^*}^2 \ge \alpha )
    & \le 2 \exp\left( - \frac{\alpha}{4p\sigma} \right)
      \le \zeta
    \enspace.
\end{align}
The error obtained by output perturbation is thus upper bounded by $\norm{h^{priv} - h^*}^2 \le 4p\sigma^2\log(2/\zeta) =  \frac{8p\Delta^2 \log(1.25/\delta) \log(2/\zeta)}{\epsilon^2}$ with probability at least $1-\zeta$.

\textbf{Estimating the Sensitivity.}
Define $g(h) = \frac{1}{n} \sum_{i=1}^n \ell(w; d_i')$ with $d_i' \in \spaceX \times \spaceY$ such that $d_i' = d_i$ for all $i \neq 1$. By strong convexity, the two following inequalities hold for $h, h'$,
\begin{align}
    f(h)
    & \ge f(h')
      + \langle \nabla f(h'), h - h' \rangle
      + \frac{\mu}{2} \| h - h' \|^2 \enspace, \\
    f(h') & \ge f(h)
      + \langle \nabla f(h), h' - h \rangle
      + \frac{\mu}{2} \| h - h' \|^2 \enspace.
\end{align}
Summing these two inequalities give $\langle \nabla f(h) - \nabla f(h'), h - h' \rangle \ge \frac{\mu}{2} \norm{h - h'}^2$. Let $h_1^*$ and $h_2^*$ be the respective minimizers of $f$ and $g$ over $\cC$, taking $h=h_1^*$ and $h'=h_2^*$ gives
\begin{align}
\label{eq:h1-minus-h2}
    \mu \| h_1^* - h_2^* \|^2
    & \le \langle \nabla f(h_1^*) - \nabla f(h_2^*), h_1^* - h_2^* \rangle
      \le \| \nabla f(h_1^*) - \nabla f(h_2^*) \| \cdot \| h_1^* - h_2^* \| \enspace.
\end{align}
Now, optimality conditions give
\begin{align}
\label{eq:link-f1-f2-unconstrained}
    \nabla f(h_1^*) = 0 = \nabla g(h_2^*) = \nabla f(h_2^*) - \nabla \ell(h_2^*; d_1) + \nabla \ell(h_2^*; d_1') \enspace,
\end{align}
resulting in $\| \nabla f(h_1^*) - \nabla f(h_2^*) \| = \| \frac{1}{n} \nabla \ell(h_2^*; d_1) - \frac{1}{n} \nabla \ell (h_2^*; d_1') \| \le \frac{2\Lambda}{n}$. Combined with \eqref{eq:h1-minus-h2}, this shows that the sensitivity of $\argmin_{h \in \cC} f(h)$ is $ \Delta = \frac{2\Lambda}{n\mu}$, which concludes the proof.
\end{proof}

\section{Convergence of DP-SGD (Proof of \Cref{lemma:sensitivity-dp-sgd})}
\label{app:sec:convergence-dp-sgd}

\begin{lemma*}
  Let $h^{\priv}$ be the vector released by DP-SGD with
  $\sigma^2 = \sfrac{64 \Lambda^2 T^2 \log(3T/\delta) \log(2/\delta) }{ n^2
  \epsilon^2 }$. Assume that
  $\sigma_*^2 = \expect_{i\sim[n]} \norm{\nabla \ell(h^*; x_i, y_i)}^2
  \le \sigma^2$. Let $0 < \zeta < 1$, then with probability at least
  $1 - \zeta$,
  \begin{align*}
    \norm{h^{\priv} - h^*}_2^2
    & = \widetilde O\left(  \frac{p \Lambda^2 \log(1/\delta)^2}{\zeta \mu^2 n^2 \epsilon^2} \right)
      \enspace,
  \end{align*}
  where $\widetilde O$ ignores logarithmic terms in $n$ (the number of examples) and $p$ (the number of model parameters).
\end{lemma*}

\begin{proof}
We start by recalling that in DP-SGD,
\begin{align}
  h^{t+1} = \pi_{\cH}(h^t - \gamma ( g^t + \eta^t ))
  \enspace.
\end{align}
Since $h^* \in \cH$, and $\cH$ is convex, we have
\begin{align}
  \norm{h^{t+1} - h^*}^2
  & = \norm{\pi_{\cH}(h^t - \gamma ( g^t + \eta^t )) - h^*}^2 \\
  & = \norm{h^t - h^*}^2
  - 2\gamma \scalar{g^t + \eta^t}{h^t - h^*}
  + \gamma^2 \norm{g^t + \eta^t}^2 \\
  & \le \norm{h^t - h^*}^2
  - 2\gamma \scalar{g^t + \eta^t}{h^t - h^*}
  + 2\gamma^2 \norm{g^t}^2
    + 2\gamma^2 \norm{\eta^t}^2
    \enspace,
\end{align}
where we developed the square and used
$\norm{a+b}^2 \le 2\norm{a}^2 + 2\norm{b}^2$ for $a,b \in
\RR^p$. Taking the expectation with respect to the stochastic gradient
computation and noise, we obtain
\begin{align}
  \expect \norm{h^{t+1} - h^*}^2
  & \le \norm{h^t - h^*}^2
  - 2\gamma \scalar{\nabla f(h^t)}{h^t - h^*}
  + 2\gamma^2 \expect\norm{g^t}^2
    + 2\gamma^2 \expect\norm{\eta^t}^2
    \enspace,
\end{align}
since $\expect(\eta^t) = 0$ and $\expect(g^t) = \nabla f(h^t)$. Now
recall that, by strong-convexity of $f$, we have
\begin{align}
  f(h^*)
  & \ge f(h^t)
    + \scalar{\nabla f(h^t)}{h^* - h^t}
    + \frac{\mu}{2} \norm{h^t - h^*}^2
    \enspace.
\end{align}
By reorganizing, we obtain
$-2\gamma \scalar{\nabla f(h^t)}{h^t - h^*} \le - 2 \gamma(f(h^t) -
f(h^*)) - \gamma \mu \norm{h^t - h^*}^2$, which gives
\begin{align}
  \expect\norm{h^{t+1} - h^*}^2
  & \le (1 - \gamma \mu) \norm{h^t - h^*}^2
    - 2 \gamma(f(h^t) - f(h^*))
    + 2\gamma^2 \expect\norm{g^t}^2
    + 2\gamma^2 \expect\norm{\eta^t}^2
    \enspace.
\end{align}
Finally, remark that if $f = \frac{1}{n} \sum_{i=1}^n f_i$ with each
$f_i$ being $\beta$-smooth and $\expect f_i = f$, we have, for
$i \sim [n]$,
\begin{align}
  \expect\norm{\nabla f_i(h^t)}^2
  & = \expect\norm{\nabla f_i(h^t) - \nabla f_i(h^*) + \nabla f_i(h^*)}^2 \\
  & \le \expect(2\norm{\nabla f_i(h^t) - \nabla f_i(h^*)}^2
    + 2\norm{\nabla f_i(h^*)}^2) \\
  & \le \expect(4\beta(f_i(h^t) - f_i(h^*) - \scalar{\nabla f_i(h^*)}{h^t - h^*})
    + 2\norm{\nabla f_i(h^*)}^2) \\
  & = 4\beta(f(h^t) - f(h^*))
    + 2\expect\norm{\nabla f_i(h^*)}^2
    \enspace,
\end{align}
since $f_i$ is $\beta$-smooth, which implies, for all $w, v \in \RR^p$,
\begin{align}
  \norm{\nabla f_i(w) - \nabla f_i(v)}^2
  & \le 2 \beta(f_i(w) - f_i(v) - \scalar{\nabla f_i(v)}{w - v}
    \enspace,
\end{align}
and $\expect \nabla f_i(h^*) = 0$. Combined with the fact that
$\expect\norm{\nabla f_i(h^*)}^2 \le \sigma_*^2$ and
$\expect\norm{\eta^t}^2 = p \sigma^2$, we obtained
\begin{align}
  \expect\norm{h^{t+1} - h^*}^2
  & \le (1 - \gamma \mu) \norm{h^t - h^*}^2
    + (4\beta\gamma^2 - 2 \gamma) ( f(h^t) - f(h^*))
    + 2\gamma^2 (\sigma_*^2 + \sigma^2) \\
  & \le (1 - \gamma \mu) \norm{h^t - h^*}^2
    + 4\gamma^2 \sigma^2
    \enspace,
\end{align}
since $\gamma \le 1/2\beta$, which implies $4\beta\gamma^2 - 2\gamma \le 0$
and $\sigma^* \le \sigma$. By induction, we obtain that, after
$T$ iterations,
\begin{align}
  \expect\norm{h^T - h^*}^2
  & \le (1 - \gamma\mu)^T \norm{h^0 - h^*}^2
    + 4 \gamma^2 \sum_{t=0}^{T-1} (1 - \gamma\mu)^{T-t} \sigma^2 \\
  & \le (1 - \gamma\mu)^T \norm{h^0 - h^*}^2
    + \frac{4 \gamma \sigma^2}{\mu}
    \enspace.
\end{align}
Now, recall that DP-SGD is $(\epsilon,\delta)$-differentially private for
$\sigma^2 = \frac{64 \Lambda^2 T \log(3T/\delta) \log(2/\delta)}{n^2
  \epsilon^2}$ (following from the Gaussian mechanism, advanced
composition theorem and amplification by subsampling). Thus, taking $\gamma = 1/2\beta$, and setting $T = \frac{2\beta}{\mu} \log(\mu \beta \norm{h^0-h^*}^2 / 2M^2)$,
where $M^2 = \frac{64\Lambda^2 \log(2/\delta)}{n^2\epsilon^2}$, yields
\begin{align}
  \expect \norm{h^T - h^*}^2
  & \le \frac{2 (T \log(3T/\delta) +1) M^2}{\beta\mu}
  \le \frac{8 M^2}{\mu^2} \log\Big(\frac{\mu \beta \norm{h^0-h^*}^2}{2M^2}\Big)
    \log\Big(\frac{6\beta \log(\tfrac{\mu \beta \norm{h^{0}-h^*}^2}{ 2M^2})}{ \mu \delta}\Big)
    \enspace.
\end{align}
Using Markov inequality, we obtain
\begin{align}
  \prob \left(
  \norm{h^T - h^*}^2 \ge
  \frac{8 M^2}{\zeta \mu^2} \log\Big(\frac{\mu \beta \norm{h^0-h^*}^2}{2M^2}\Big)
    \log\Big(\frac{6\beta \log(\tfrac{\mu \beta \norm{h^{0}-h^*}^2}{ 2M^2})}{ \mu \delta}\Big)
  \right)
  \le \zeta
    \enspace.
\end{align}
This results in
the following upper bound, with probability at least $1 - \zeta$,
\begin{align}
  \norm{h^T - h^*}^2
  & \le \frac{512 \Lambda^2 \log(2/\delta)}{\zeta \mu^2 n^2 \epsilon^2}
  \log\Big(\frac{\mu \beta \norm{h^0-h^*}^2}{2M^2}\Big)
    \log\Big(\frac{6\beta \log(\tfrac{\mu \beta \norm{h^{0}-h^*}^2}{ 2M^2})}{ \mu \delta}\Big)
  \\
  & = \widetilde O\left( \frac{G^2 \log(1/\delta)}{\zeta \mu^2 n^2\epsilon^2} \right)
    \enspace,
\end{align}
which is the result of our lemma.
\end{proof}

\section{Additional Experimental Details}
\label{app:sec:additional-experiments}

\subsection{Experimental Setup}

The \celebA dataset \citep{liu2015faceattributes} is a face attributes dataset, that can be downloaded at \url{http://mmlab.ie.cuhk.edu.hk/projects/CelebA.html}, and the \folktables dataset \citep{ding2021retiring} is derived from US Census, and can be downloaded using a Python package available here \url{https://github.com/zykls/folktables}.

On each dataset, for each value of $n$, we train a $\ell_2$-regularized logistic regression model using \texttt{scikit-learn} \citep{scikit-learn}. Private models are then learned using the output perturbation mechanism as described in \Cref{sec:two-priv-mech}. We then compute our bounds using the non-private model as reference, over a test set containing $10\%$ of the data, that has not been used for training (containing $20,260$ records for \celebA and $166,450$ records for \folktables). The value of the bound is computed by minimizing the experession given by the Chernoff bound using the golden section search algorithm \citep{kiefer1953sequential}. The code is in the supplementary, and will be made public.

For the plots with different number of training records, we train $20$ non-private models with a number of records logarithmically spaced between $10$ and the number of records in the complete training set (that is, $182,339$ for \celebA and $1,498,050$ for \folktables). For the plots with different privacy budgets, we use $20$ values logarithmically spaced between $10^{-3}$ and $10$ for both datasets.

\subsection{Results for Other Fairness Measures}

Our bounds also hold for accuracy parity, demographic parity and equalized odds. The same plots as those presented in \Cref{fig:fairness-fct-n-epsilon} for these fairness notions are in \Cref{app:fig:fairness-fct-n} and \Cref{app:fig:fairness-fct-epsilon}. The comments from \Cref{sec:experiments} on equality of opportunity and accuracy also hold for these three notions of fairness.

\begin{figure*}[h]
  \centering
  \begin{subfigure}{\textwidth}
    \centering
    \includegraphics[width=0.9\linewidth]{rsc/legend.pdf}
    \vspace{-0.2em}
  \end{subfigure}
  \begin{subfigure}{0.33\textwidth}
      \centering
      \includegraphics[width=0.725\linewidth]{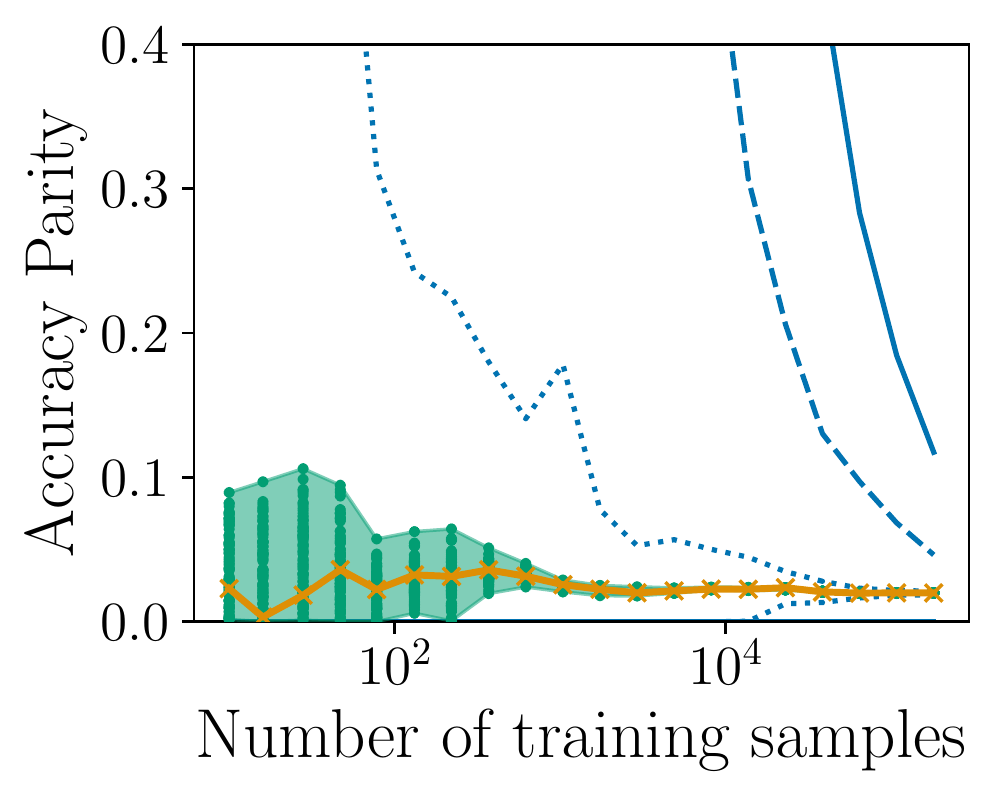}
    \caption{Accuracy Parity (\celebA)}
    \label{app:fig:fairness-fct-n-f11}
  \end{subfigure}
  \hfill
  \begin{subfigure}{0.33\textwidth}
      \centering
    \includegraphics[width=0.725\linewidth]{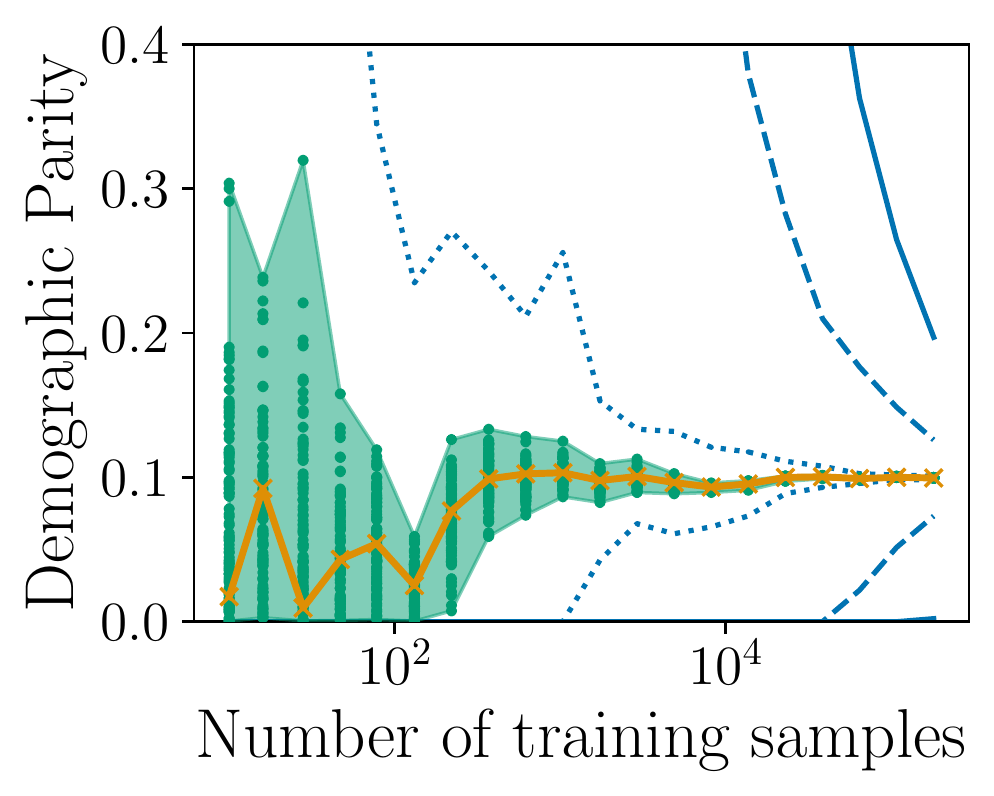}
    \caption{Demographic Parity (\celebA)}
    \label{app:fig:fairness-fct-n-a11}
  \end{subfigure}
  \hfill
  \begin{subfigure}{0.33\textwidth}
      \centering
    \includegraphics[width=0.725\linewidth]{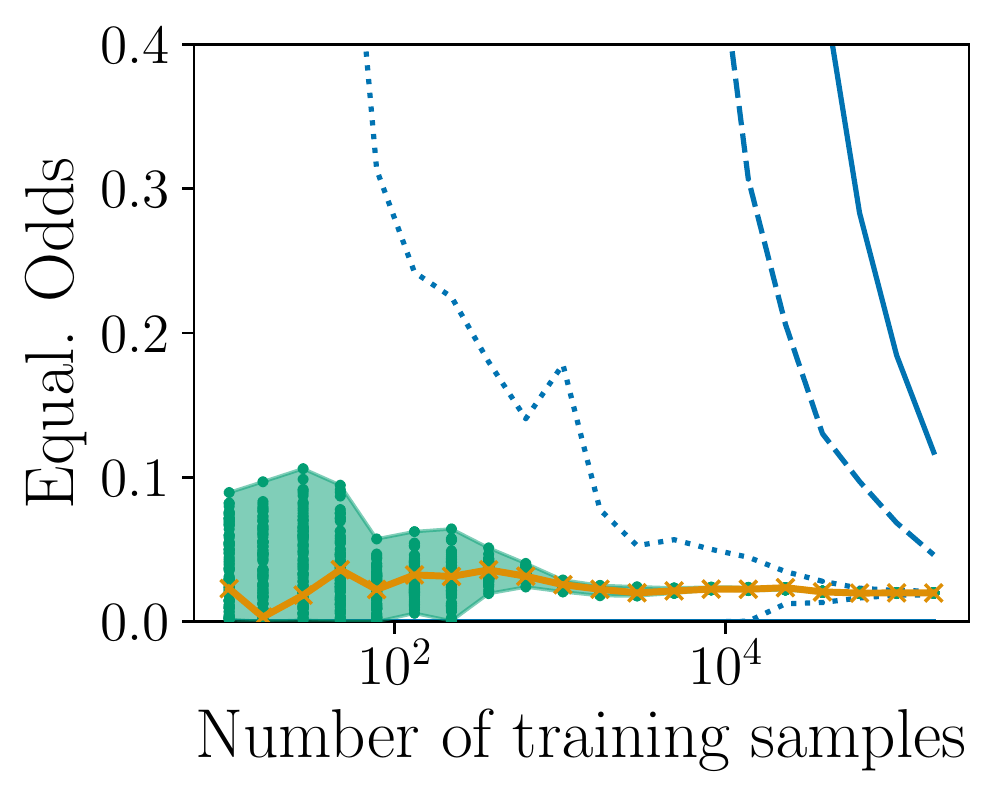}
    \caption{Equalized Odds (\celebA)}
    \label{app:fig:fairness-fct-n-f21}
  \end{subfigure}

  \begin{subfigure}{0.33\textwidth}
      \centering
    \includegraphics[width=0.725\linewidth]{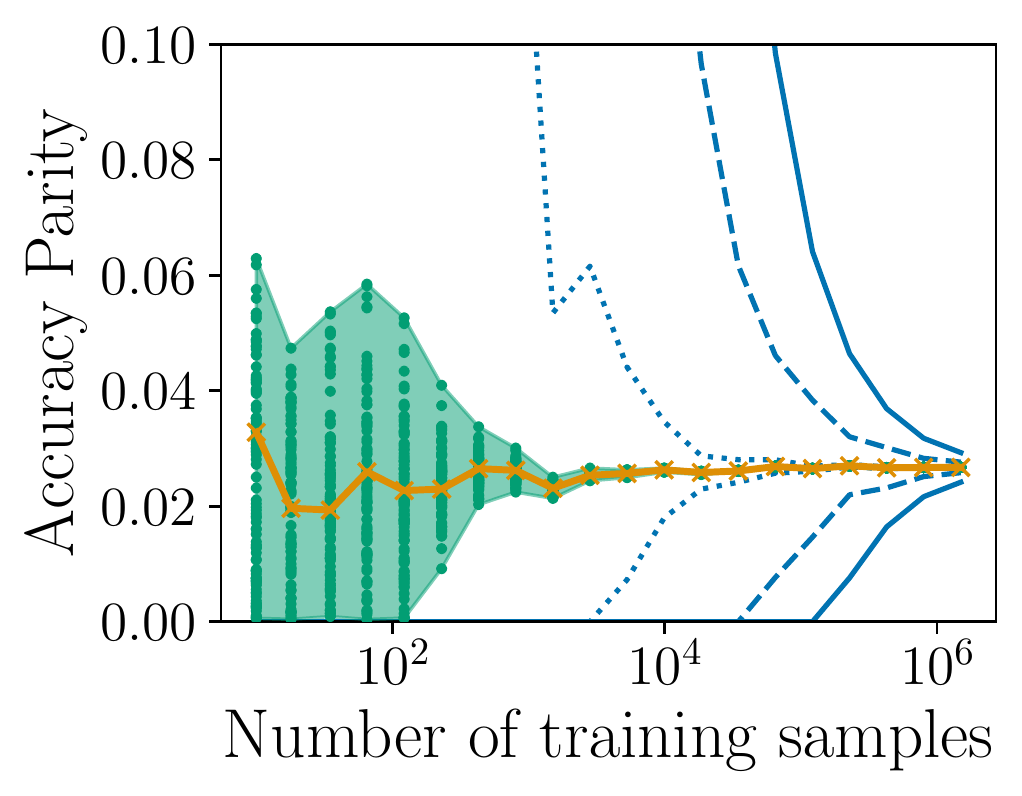}
    \caption{Accuracy Parity (\folktables)}
    \label{app:fig:fairness-fct-n-f12}
  \end{subfigure}
  \hfill
  \begin{subfigure}{0.33\textwidth}
      \centering
    \includegraphics[width=0.725\linewidth]{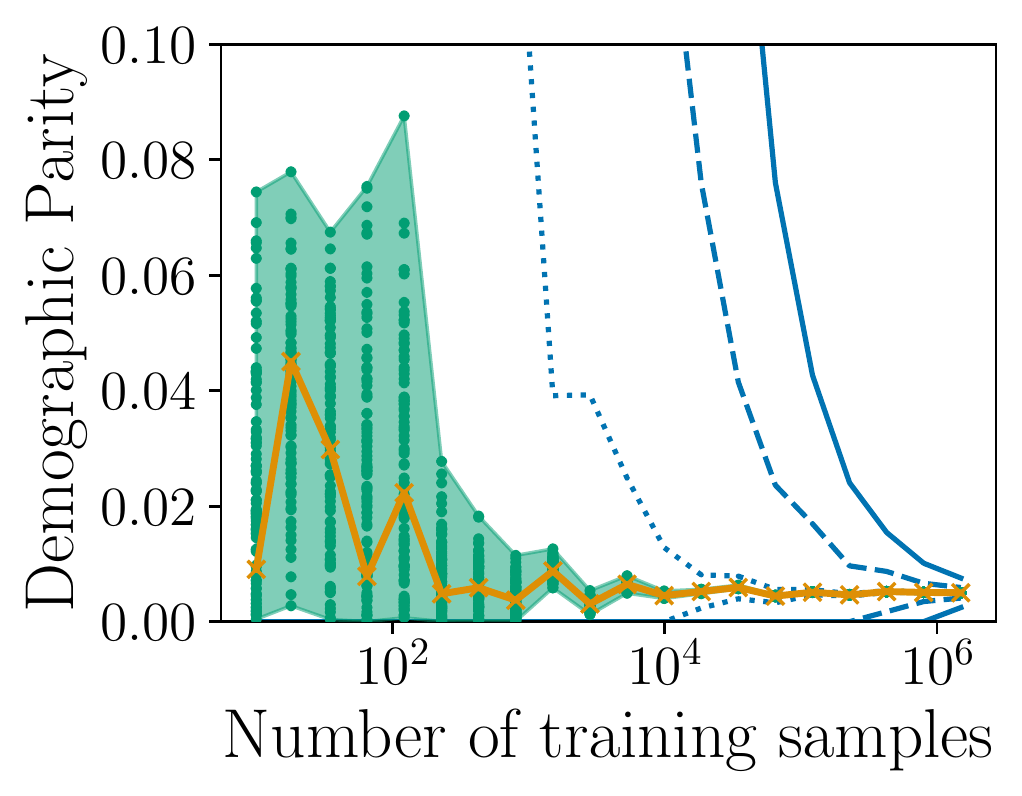}
    \caption{Demographic Parity (\folktables)}
    \label{app:fig:fairness-fct-n-a12}
  \end{subfigure}
  \hfill
  \begin{subfigure}{0.33\textwidth}
      \centering
    \includegraphics[width=0.725\linewidth]{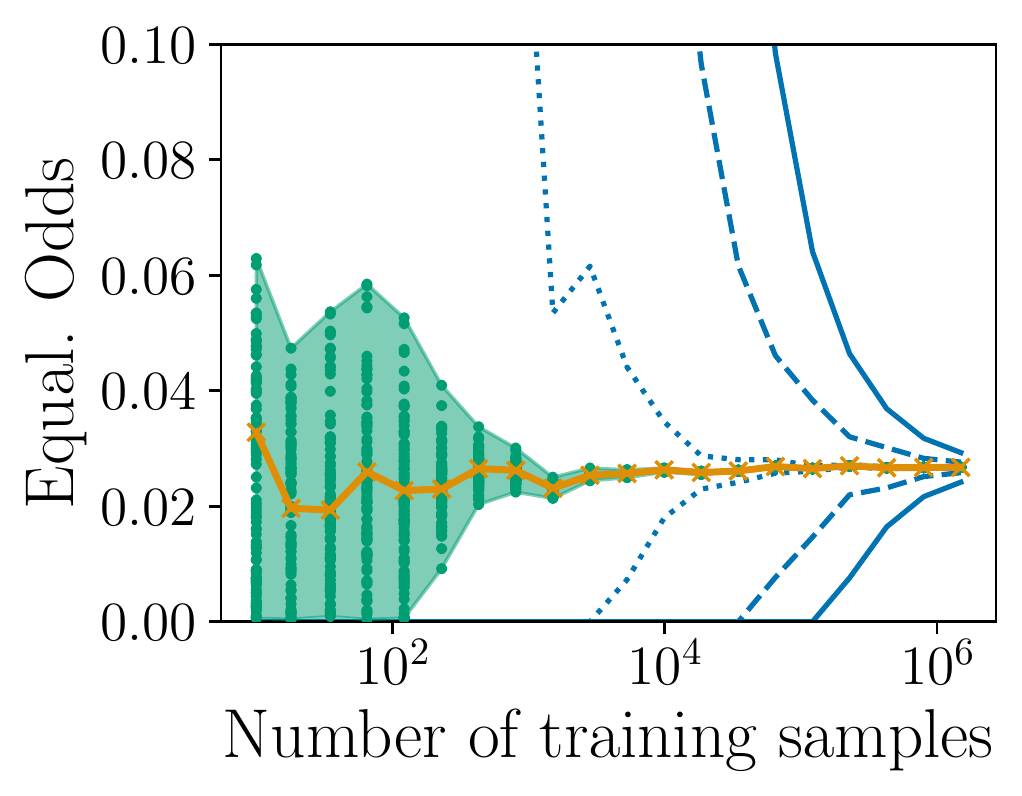}
    \caption{Equalized Odds (\folktables)}
    \label{app:fig:fairness-fct-n-f22}
  \end{subfigure}
  \caption{Fairness and accuracy levels for optimal non-private model
    and random private ones as a function of the number $n$ of
    training samples. For each value of $n$, we sample $100$ private
    models and take their minimum and maximum fairness/accuracy values
    to mark the area of attainable values. The solid blue line and the
    dashed one give our guarantees, respectively from
    \Cref{thm:bound-on-fairness-private-models} with
    \Cref{lemma:sensitivity-output-perturbation}'s bounds and with an
    empirical evaluation of $\norm{h^{\text{priv}}-h^*}$.}
  \label{app:fig:fairness-fct-n}
\end{figure*}
\begin{figure*}[h]
  \centering
  \begin{subfigure}{\textwidth}
    \centering
    \includegraphics[width=0.9\linewidth]{rsc/legend.pdf}
    \vspace{-0.2em}
  \end{subfigure}
  \begin{subfigure}{0.33\textwidth}
      \centering
    \includegraphics[width=0.725\linewidth]{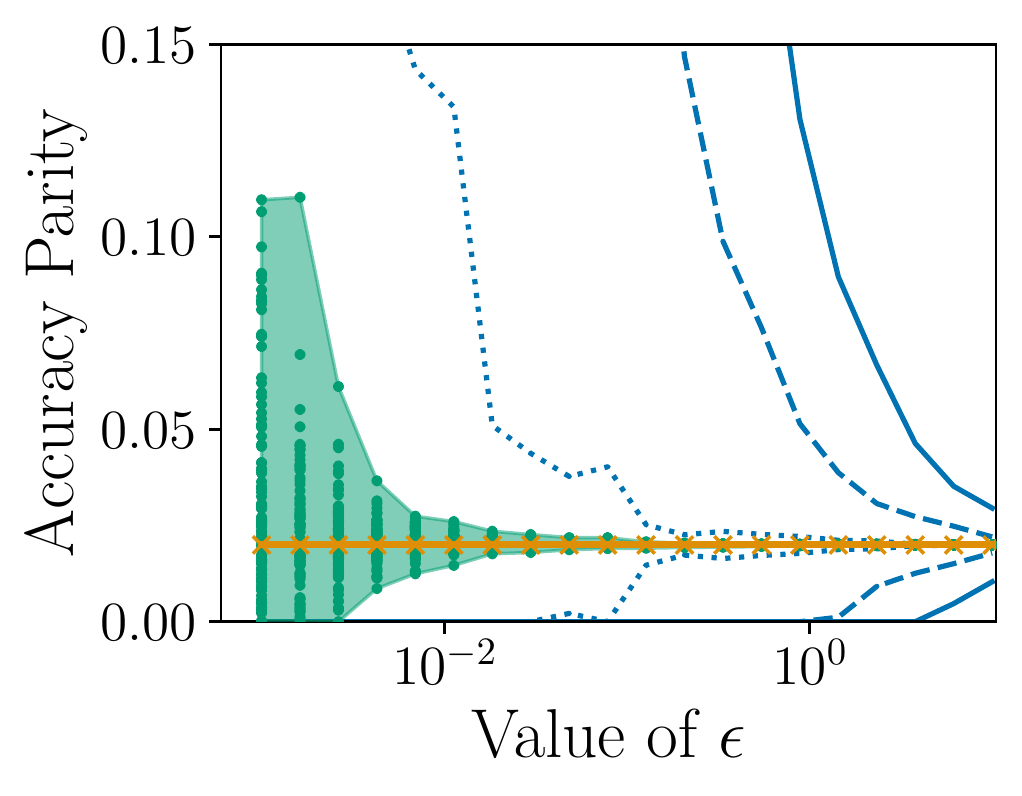}
    \caption{Accuracy Parity (\celebA)}
    \label{app:fig:fairness-fct-epsilon-f31}
  \end{subfigure}
  \hfill
  \begin{subfigure}{0.33\textwidth}
      \centering
    \includegraphics[width=0.725\linewidth]{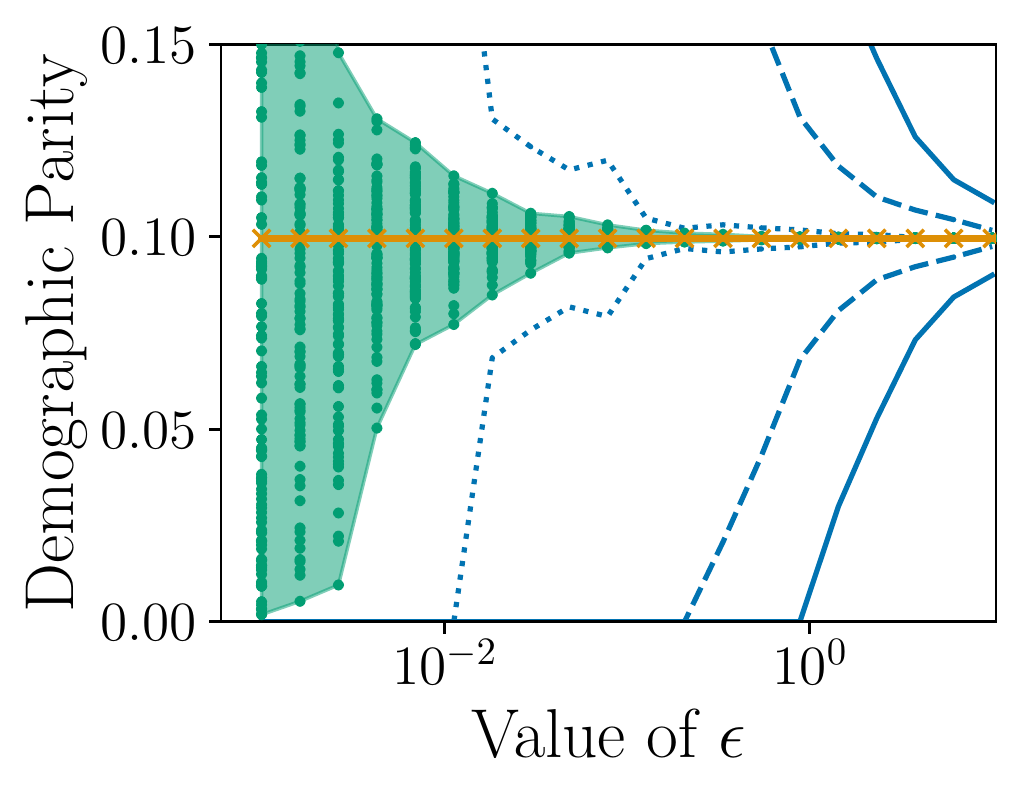}
    \caption{Demographic Parity (\celebA)}
    \label{app:fig:fairness-fct-epsilon-f41}
  \end{subfigure}
  \hfill
  \begin{subfigure}{0.33\textwidth}
      \centering
    \includegraphics[width=0.725\linewidth]{rsc/plots/fairness_fct_epsilon_equality_opportunity_fairness_fct_epsilon_celebA_1_logscale.pdf}
    \caption{Equality of Opportunity (\celebA)}
    \label{app:fig:fairness-fct-epsilon-f51}
  \end{subfigure}

  \begin{subfigure}{0.33\textwidth}
      \centering
    \includegraphics[width=0.725\linewidth]{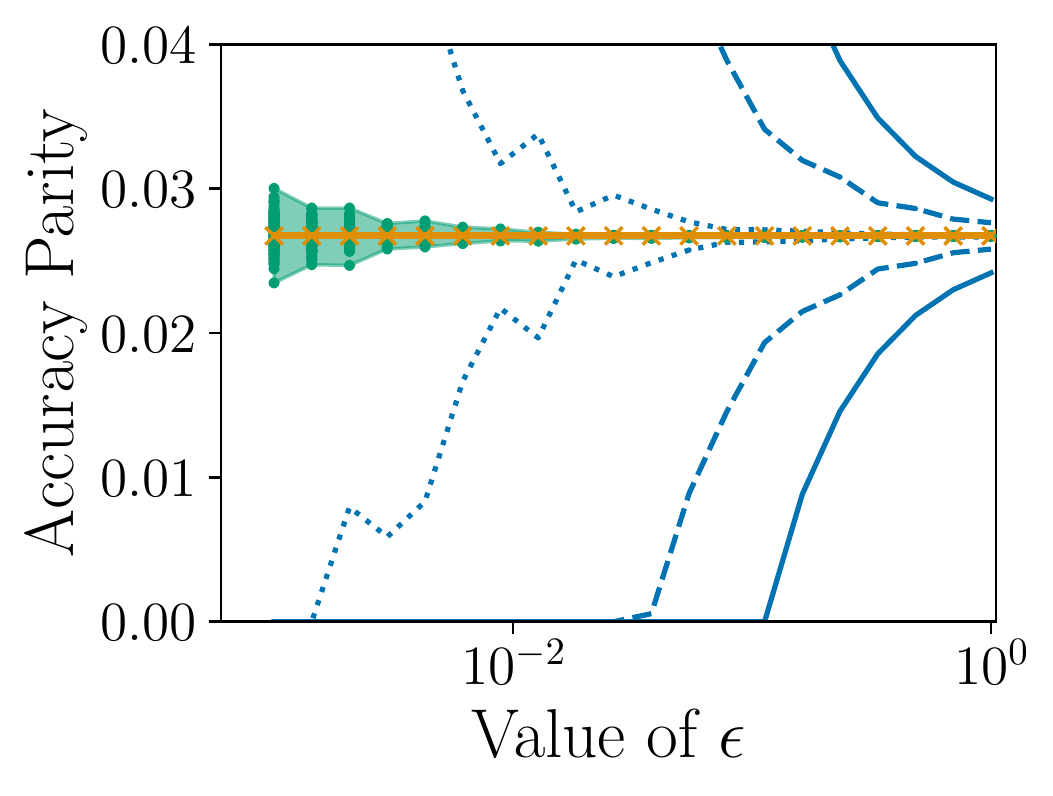}
    \caption{Accuracy Parity (\folktables)}
    \label{app:fig:fairness-fct-epsilon-f32}
  \end{subfigure}
  \hfill
  \begin{subfigure}{0.33\textwidth}
      \centering
    \includegraphics[width=0.725\linewidth]{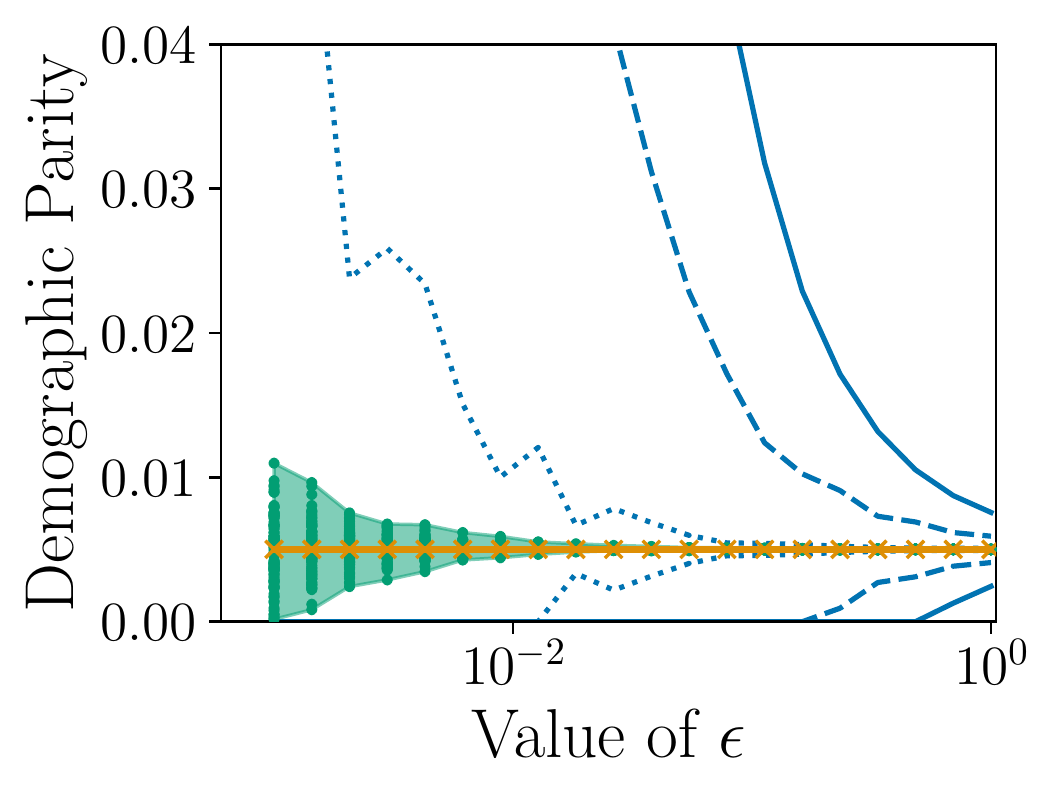}
    \caption{Demographic Parity (\folktables)}
    \label{app:fig:fairness-fct-epsilon-f42}
  \end{subfigure}
  \hfill
  \begin{subfigure}{0.33\textwidth}
      \centering
    \includegraphics[width=0.725\linewidth]{rsc/plots/fairness_fct_epsilon_equality_opportunity_fairness_fct_epsilon_folktables_1_logscale.pdf}
    \caption{Equality of Opportunity (\folktables)}
    \label{app:fig:fairness-fct-epsilon-f52}
  \end{subfigure}
  \caption{Fairness and accuracy levels for optimal non-private model
    and random private ones as a function of privacy budget
    $\epsilon$. For each value of $\epsilon$, we sample $100$ private
    models and take their minimum and maximum fairness/accuracy values
    to mark the area of attainable values. The solid blue line and the
    dashed one respectively give our guarantees, respectively from
    \Cref{thm:bound-on-fairness-private-models} with
    \Cref{lemma:sensitivity-output-perturbation}'s bounds and with an
    empirical evaluation of $\norm{h^{\text{priv}}-h^*}$.}
  \label{app:fig:fairness-fct-epsilon}
\end{figure*}

\subsection{Refined Bounds with Additional Knowledge of $\bm{h^\priv}$ and $\bm{h^*}$}
\label{app:sec:refined-bounds}

In \Cref{assum:lipschitz-constant-fct-val}, we use a uniform Lipschitz bound for all $h, h' \in \cH$. Let's consider the class $\cH$ of linear models, where, for $h \in \cH$, we denote by $h_y$ the parameters of $h$ associated with the label $y$, that is $h(x,y) = h_y^T x$.
For linear models, we derived the bound $\norm{\rho(h,x,y) - \rho(h', x, y)}_\cH \le 2\norm{x}_2 \norm{h-h'}_\cH$, as derived in \Cref{sec:preliminaries}. Note that this inequality can be very loose whenever $x$ and $h_y-h_y'$ (for $y \in \cY$) are (close to) orthogonal. When they are orthogonal, this bounds only gives $0 = (h_y-h_y')^T x \le \norm{h_y-h_y'}_2\norm{x}_2$. We can thus improve the inequality by remarking that we have
\begin{align*}
    \abs{\rho(h,x,y) - \rho(h',x,y)}
    \leq{}& \abs{h(x,y) - h'(x,y)} + \max_{y' \neq y} \abs{h(x,y') - h'(x,y')} \\
    ={}& \abs{h_y^Tx - h_y'^Tx} + \max_{y' \neq y} \abs{h_{y'}^T x - h'_{y'} x} \\
    ={}& \abs{(h_y - h_y')^T x} + \max_{y' \neq y} \abs{(h_{y'} - h'_{y'}) x} \\
    ={}& \abs{(h_y - h_y')^T p_{h_y-h'_y}(x)} + \max_{y' \neq y} \abs{(h_{y'} - h'_{y'}) p_{h_y-h'_y}(x)} \\
    \leq{}& 2 \max_{y'\in\cY} \norm{p_{h_y-h_y'}(x)} \norm{h - h'}_\cH
    \enspace,
\end{align*}
where $p_{h_y-h'_y}(x)$ is the projection of $x$ on the axis defined by $h_y - h'_y$. We can thus define a variant of $L_{X,Y}$ which depends on $h-h'$
\begin{align}
    L_{X,Y}^{h-h'}
    = 2 \max_{y\in\cY} \norm{p_{h_y-h_y'}(x)}
    \enspace.
\end{align}

Replacing \Cref{assum:lipschitz-constant-fct-val} by this inequality in the proof of \Cref{thm:bound-on-diff-proba}, we end up with the inequality
\begin{align*}
    \abs{\prob\left(H(X) = Y \;\middle|\; E\right) - \prob\left(H'(X) = Y \;\middle|\; E\right)} \leq{} \prob\left(\frac{\abs{\rho(h,X,Y)}}{L_{X,Y}^{h-h'}} \leq \norm{h - h'}_\cH \;\middle|\; E\right)
    \enspace,
\end{align*}
where the probability is over $(X,S,Y) \sim \cD$. We obtained the same bound as \Cref{thm:bound-on-diff-proba}, except with $L_{X,Y}^{h-h'}$ instead of $L_{X,Y}$. Note that even if this gives a much tighter bound, this can generally not be computed, as one of $h$ or $h'$ is typically not known.

\end{document}

%% file: usercommands.tex
\usepackage{amsmath,amsthm,amssymb}
\usepackage{cleveref}
\usepackage{comment}
\usepackage{xspace}
\usepackage{bm}

\newcommand{\celebA}{\texttt{celebA}\xspace}
\newcommand{\folktables}{\texttt{folktables}\xspace}

\allowdisplaybreaks

\newcommand{\loss}[1]{\ell\ifthenelse{\isempty{#1}{}}{}{\!\left(#1\right)}}
\newcommand{\negspace}[1]{\phantom{\hspace*{-#1}}}

\DeclareMathOperator*{\argmin}{arg\,min}
\DeclareMathOperator*{\argmax}{arg\,max}
\DeclareMathOperator*{\expect}{\mathbb{E}}
\DeclareMathOperator*{\prob}{\mathbb{P}}

\newcommand{\mydefv}[1]{\expandafter\newcommand\csname v#1\endcsname{\mathbf{#1}}}
\newcommand{\mydefallv}[1]{\ifx#1\mydefallv\else\mydefv{#1}\expandafter\mydefallv\fi}
\mydefallv abcekmosuvwxyz\mydefallv

\newcommand{\mydefvsym}[1]{\expandafter\newcommand\csname v#1\endcsname{\boldsymbol{\csname #1\endcsname}}}
\newcommand{\mydefallvsym}[1]{\ifx#1\mydefallvsym\else\mydefvsym{#1}\expandafter\mydefallvsym\fi}
\mydefallvsym {sigma}{alpha}{gamma}{mu}{rho}\mydefallvsym

\newcommand{\mydefset}[1]{\expandafter\newcommand\csname set#1\endcsname{{#1}}}
\newcommand{\mydefallset}[1]{\ifx#1\mydefallset\else\mydefset{#1}\expandafter\mydefallset\fi}
\mydefallset BFHPSTVYZ\mydefallset

\newcommand{\mydefdistr}[1]{\expandafter\newcommand\csname distr#1\endcsname{\mathcal{D}_{\csname space#1\endcsname}}}
\newcommand{\mydefalldistr}[1]{\ifx#1\mydefalldistr\else\mydefdistr{#1}\expandafter\mydefalldistr\fi}
\mydefalldistr DHNSTVXZ\mydefalldistr

\newcommand{\mydefedistr}[1]{\expandafter\newcommand\csname edistr#1\endcsname{\widehat{\mathcal{D}}_{\csname space#1\endcsname}}}
\newcommand{\mydefalledistr}[1]{\ifx#1\mydefalledistr\else\mydefedistr{#1}\expandafter\mydefalledistr\fi}
\mydefalledistr WZ\mydefalledistr

\newcommand{\mydefspace}[1]{\expandafter\newcommand\csname space#1\endcsname{\mathcal{#1}}}
\newcommand{\mydefallspace}[1]{\ifx#1\mydefallspace\else\mydefspace{#1}\expandafter\mydefallspace\fi}
\mydefallspace DFGHKLMNPSRTUVXYZ\mydefallspace

\newcommand{\mydeff}[1]{\expandafter\newcommand\csname f#1\endcsname[2][]{#1##1\ifthenelse{\equal{##2}{}}{}{\!\left(##2\right)}}}
\newcommand{\mydefallf}[1]{\ifx#1\mydefallf\else\mydeff{#1}\expandafter\mydefallf\fi}
\mydefallf cdfqghklsuyzCFGHKLMRT\mydefallf

\newcommand{\mydeffsym}[1]{\expandafter\newcommand\csname f#1\endcsname[2][]{\csname #1\endcsname##1\ifthenelse{\equal{##2}{}}{}{\!\left(##2\right)}}}
\newcommand{\mydefallfsym}[1]{\ifx#1\mydefallfsym\else\mydeffsym{#1}\expandafter\mydefallfsym\fi}
\mydefallfsym {Omega}{phi}{epsilon}{delta}{eta}{theta}{nu}{tildeh}{kappa}{rho}\mydefallfsym

\newcommand{\mydefnset}[1]{\expandafter\newcommand\csname nset#1\endcsname{\mathbb{#1}}}
\newcommand{\mydefallnset}[1]{\ifx#1\mydefallnset\else\mydefnset{#1}\expandafter\mydefallnset\fi}
\mydefallnset CNRSZ\mydefallnset

\newcommand{\norm}[1]{\left\|#1\right\|}
\newcommand{\abs}[1]{\left|#1\right|}
\newcommand{\babs}[1]{\big|#1\big|}

\newcommand{\scalar}[2]{\langle #1, #2 \rangle}

\theoremstyle{plain}

\newtheorem*{myth*}{Theorem}

\newtheorem*{mycor*}{Corollary}

\newtheorem*{mylem*}{Lemma}

\newtheorem{myex}{Example}
\newtheorem*{myex*}{Example}

\newtheorem*{myrmq*}{Remark}

\makeatletter
\def\th@plain{%
  \thm@notefont{}
  \itshape 
}
\def\th@definition{%
  \thm@notefont{}
  \normalfont 
}
\makeatother

\newcommand{\eDDPm}[1][]{\widehat{\text{DDP}}^+}
\newcommand{\eDDPp}[1][]{\widehat{\text{DDP}}^-}

\newcommand{\eDDPkappam}[1][]{\widehat{\text{DDP}}_{\kappa}^-}
\newcommand{\eDDPdeltap}[1][]{\widehat{\text{DDP}}_{\delta}^+}

\newcommand{\cause}[1]{&{\color{gray}\downarrow{}\;{\small \text{#1}}} \nonumber\\}

\newcommand{\priv}{\mathrm{priv}}
\newcommand{\refer}{\mathrm{ref}}

\DeclareMathOperator*{\Acc}{Acc}

\usepackage{csvsimple}
\usepackage{array,booktabs}

\makeatletter
\csvset{
  autotabularcenter/.style={
    file=#1,
    after head=\csv@pretable\begin{tabular}{|*{\csv@columncount}{c|}}\csv@tablehead,
    table head=\hline\csvlinetotablerow\\\hline,
    late after line=\\,
    table foot=\\\hline,
    late after last line=\csv@tablefoot\end{tabular}\csv@posttable,
    command=\csvlinetotablerow},
  autobooktabularcenter/.style={
    file=#1,
    after head=\csv@pretable\begin{tabular}{*{\csv@columncount}{c}}\csv@tablehead,
    table head=\toprule\csvlinetotablerow\\\midrule,
    late after line=\\,
    table foot=\\\bottomrule,
    late after last line=\csv@tablefoot\end{tabular}\csv@posttable,
    command=\csvlinetotablerow},
}
\makeatother

\newcommand{\csvautobooktabularcenter}[2][]{\csvloop{autobooktabularcenter={#2},#1}}

%% file: letters.tex
\newcommand{\cA}{\mathcal{A}}

\newcommand{\cC}{\mathcal{C}}
\newcommand{\cD}{\mathcal{D}}

\newcommand{\cH}{\mathcal{H}}

\newcommand{\cN}{\mathcal{N}}

\newcommand{\cS}{\mathcal{S}}

\newcommand{\cU}{\mathcal{U}}

\newcommand{\cX}{\mathcal{X}}
\newcommand{\cY}{\mathcal{Y}}

\newcommand{\RR}{\mathbb{R}}

\newcommand{\ie}{{\em i.e.,~}}

\newcommand{\eg}{{\em e.g.,~}}
